\documentclass[11pt]{article}

\usepackage[final]{acl}

\usepackage{times}
\usepackage{latexsym}

\usepackage[T1]{fontenc}
\usepackage[utf8]{inputenc}

\usepackage{microtype}

\usepackage{inconsolata}
\usepackage{multirow}

\usepackage{graphicx}

\usepackage{amsmath}
\usepackage{amssymb}
\usepackage{mathtools}
\usepackage{amsthm}
\usepackage{subcaption} % provides subfigure environment
\usepackage{caption}
\usepackage{enumitem}                  % For custom lists and margins

\usepackage{amsthm}
\newtheorem{theorem}{Theorem}
\newtheorem{proposition}[theorem]{Proposition}

\usepackage[capitalize,noabbrev]{cleveref}
\usepackage[dvipsnames]{xcolor}
\usepackage{pifont} % for ding symbols

\newcommand{\cmark}{\ding{51}} % ✓
\newcommand{\xmark}{\ding{55}} % ✗

\usepackage[textsize=tiny]{todonotes}

\usepackage{booktabs,array}
\usepackage[table]{xcolor}
\usepackage{caption}

\definecolor{V7WHeaderBG}{HTML}{E7D9FF}

\definecolor{NextQA}{RGB}{220,235,255}
\definecolor{DramaQ}{RGB}{223,255,224}
\definecolor{STARc}{RGB}{255,232,204}
\definecolor{SeqA}{RGB}{214,234,248}
\definecolor{SeqB}{RGB}{220,245,220}
\definecolor{SeqC}{RGB}{224,244,248}
\definecolor{SeqD}{RGB}{243,229,245}
\definecolor{accH}{RGB}{0,123,67}
\definecolor{fogH}{RGB}{130,60,145}
\definecolor{bestbg}{RGB}{255,247,173}
\definecolor{bestfg}{RGB}{139,108,0}
\definecolor{runnerbg}{RGB}{216,231,255} % light blue
\definecolor{runnerfg}{RGB}{0,76,153}    % dark blue

\newcommand{\hdr}[2]{\multicolumn{2}{c}{\cellcolor{#1}\textbf{#2}}}
\newcommand{\HACC}{\textcolor{accH}{Acc ($\uparrow$)}}
\newcommand{\ACC}{\textcolor{accH}{Acc}}
\newcommand{\HFOG}{\textcolor{fogH}{Fog ($\downarrow$)}}
\newcommand{\FOG}{\textcolor{fogH}{Fog}}
\newcommand{\best}[1]{\cellcolor{bestbg}\textcolor{bestfg}{\textbf{#1}}}
\newcommand{\runner}[1]{\cellcolor{runnerbg}\textcolor{runnerfg}{\underline{#1}}}
\newcolumntype{L}[1]{>{\raggedright\arraybackslash}p{#1}}

\newcommand{\venue}[1]{\textcolor{black!60}{\footnotesize\textsf{(#1)}}}

\newcommand{\bestcap}[1]{\colorbox{bestbg}{\textcolor{bestfg}{\textbf{#1}}}}
\newcommand{\runnercap}[1]{\colorbox{runnerbg}{\textcolor{runnerfg}{\underline{#1}}}}

\definecolor{capBlue}{RGB}{70,120,255}   % hơi sáng, nhìn "ra blue"
\definecolor{capGreen}{RGB}{40,150,90}   % xanh lá dịu
\definecolor{capOrange}{RGB}{230,150,40} % cam dịu
\definecolor{capCyan}{RGB}{0,160,180}    % cyan/teal
\definecolor{capPurple}{RGB}{150,90,190} % tím

\colorlet{V7WHdr}{teal!45!violet} % teal-violet; chỉnh 55 để xanh/tím hơn

\newcommand{\VisualWcap}{\textcolor{V7WHdr}{\textbf{Visual7W}}}

\newcommand{\NextQAcap}{\textcolor{capBlue}{\textbf{NExT\mbox{-}QA}}}
\newcommand{\DramaQcap}{\textcolor{capGreen}{\textbf{DramaQA}}}

\DeclareRobustCommand{\bestcap}[1]{%
  \colorbox{bestbg}{\textcolor{bestfg}{\textbf{#1}}}%
}
\DeclareRobustCommand{\runnercap}[1]{%
  \colorbox{runnerbg}{\textcolor{runnerfg}{\underline{#1}}}%
}

\definecolor{vqaH}{RGB}{0,110,200}   % VQA_C (blue)
\definecolor{fogC}{RGB}{180,95,0}    % Fog_C (orange/brown)

\definecolor{ClimbC}{RGB}{255,240,245} % hồng rất nhạt (ví dụ)

\title{DynaTokens: Controlling Token Dynamics for \\ Continual Video--Language Understanding}

\author{
  \textbf{Toan Nguyen}\textsuperscript{1} \quad
  \textbf{Yang Liu}\textsuperscript{1} \quad
  \textbf{Celso De Melo}\textsuperscript{2} \quad
  \textbf{Flora D. Salim}\textsuperscript{1}
\\
  \textsuperscript{1}\,School of Computer Science and Engineering, University of New South Wales, Australia \\
  \textsuperscript{2}\,DEVCOM Army Research Laboratory, USA \\
  \texttt{\{toan.nguyen,flora.salim\}@unsw.edu.au} \\
  \texttt{yang.liu39@student.unsw.edu.au},~\texttt{celso.miguel.de.melo@gmail.com} \\
  \small{\textbf{Correspondence:} \href{mailto:flora.salim@unsw.edu.au}{flora.salim@unsw.edu.au}}
}

\newcommand{\model}{\textsc{DynaTokens}}

\begin{document}
\maketitle

\begin{abstract}
Continual VideoQA with multimodal LLMs remains challenging because sequential adaptation induces task interference, while storing task-specific prompts becomes impractical as task sequences grow.
We introduce \textbf{\model}, a transformer-based token generator that dynamically produces fine-tuning tokens on demand, enabling task-adaptive prompt updates through shared generation weights.
To mitigate forgetting, we introduce meta-learning-inspired regularisers that look ahead to avoid task-specific sharp update directions while anchoring the evolving generator to prior-task behaviours.
We theoretically connect this objective to sharpness-aware optimisation, showing how it favours flatter cross-task minima and improves retention.
\model~ combines gradient-free routing based on robust pretrained token and visual embeddings with lightweight auxiliary multimodal supervision, reducing router drift during continual adaptation.
Across standard continual VideoQA benchmarks, \model~ achieves higher average accuracy and substantially lower forgetting than strong baselines.
It also improves zero-shot generalisation and remains effective in longer domain-incremental sequences with extended task shifts.
Finally, we introduce a challenging ImageQA$\rightarrow$VideoQA protocol and show that \model~ enables robust cross-modal continual transfer.
Our code is available at \href{https://github.com/cruiseresearchgroup/DynaTokens}{this link}.

% Our code is available at \href{https://anonymous.4open.science/r/HyperTokens-ICML2026-Submission-6CE2}{this link}.

% design HyperTokens well?

% theoretical part link to SAM???
\end{abstract}

\section{Introduction}
%Adapt the first part for the specific problem such as VideoQA. [Add another sentence on SOTA performance of recent models on VideoQA. However, ...]

% Modern AI increasingly operates on continuous multimodal data streams, comprising video, audio, language, and structured signals from dynamic environments (e.g., wearables, cameras and sensors). In this setting, the conventional "train-then-deploy" paradigm, even for large-scale models such as GPT-4, assumes fixed datasets and stationarity and thus degrades under distribution shift, emerging tasks, and novel concepts. Naïvely fine-tuning on new tasks induces \textit{catastrophic forgetting} \cite{mccloskey1989catastrophic,tan2025bisecle}—erasing base knowledge and prior tasks/concepts—and the effect worsens as tasks accumulate.

% Address the problem of growing the size of network

% Also, in continual VideoQA, tasks differ in both video domains and question types (e.g., counting, localisation); this heterogeneity makes the problem more challenging.

% an effort to 

Modern multimodal LLMs have demonstrated remarkable capacity for jointly reasoning over video, images, audio, and language, achieving strong performance on tasks such as VideoQA~\cite{shu2025audio, maaz2024video, qian2024streaming}. However, the prevailing \emph{train-then-deploy} paradigm~\cite{liu2024deepseek, team2024gemini, zhao2024continual, liu2025continual} rests on a stationarity assumption that breaks down in the face of evolving task distributions, domain shifts, and previously unseen concepts. Na\"ively fine-tuning on incoming tasks leads to \emph{catastrophic forgetting}~\cite{mccloskey1989catastrophic,cai2024empowering}, and full-parameter updates quickly become prohibitive at scale. Parameter-efficient adaptation (PEA)~\cite{cai2024empowering,adhikari2025adaprefix++,razdaibiedina2023progressive} provides an appealing alternative by introducing a small number of learnable parameters atop a frozen backbone, yet applying it to the continual setting remains an open challenge.

The core difficulty is the tension between \emph{plasticity} and \emph{stability}: sharply shifting task distributions cause sequential prompt updates to interfere and erode prior knowledge~\cite{tan2025bisecle}. Allocating separate adapter prompts per task~\cite{razdaibiedina2023progressive,zeng2025modalprompt} preserves task-specific knowledge but scales poorly, while sharing prompt parameters across tasks~\cite{tan2025bisecle,cai2024empowering,wang2022learning} improves efficiency yet exposes shared components to conflicting gradients~\cite{nguyen2026fogo}. Alternatives such as LoRA and prefix tuning~\cite{adhikari2025adaprefix++, yu2025moe} also fall short of simultaneously (i)~scaling gracefully to many tasks, (ii)~preserving fine-grained, task-specific control without cross-task interference, and (iii)~generalising reliably to multimodal temporal settings such as VideoQA.

% However, in continual learning for challenging multimodal video tasks such as VideoQA, task distributions can differ sharply (e.g., indoor versus outdoor videos or distinct question types), so prompt updates may interfere and still incur substantial forgetting without careful parameter control~\cite{tan2025bisecle}. Recent methods \cite{razdaibiedina2023progressive, rusu2016progressive} sidestep this by storing task-specific adapter prompts and stacking them over time, but the approach scales poorly as the number of tasks grows. Other approaches \cite{tan2025bisecle,cai2024empowering,wang2022learning} share a prompt parameterisation across tasks, but updates can conflict on the shared components and trigger interference under a small prompt budget. Alternatives such as LoRA or prefix tuning \cite{adhikari2025adaprefix++, yu2025moe} are efficient, yet they remain challenging to (i) scale \emph{nicely} to many tasks, (ii) preserve fine-grained, task-specific control \emph{without} cross-task interference, and (iii) extend reliably to multimodal VideoQA.

\begin{figure*}[t]
  \centering
  \setlength{\tabcolsep}{4pt}
  \begin{tabular}{ccc}
    \includegraphics[width=0.59\textwidth]{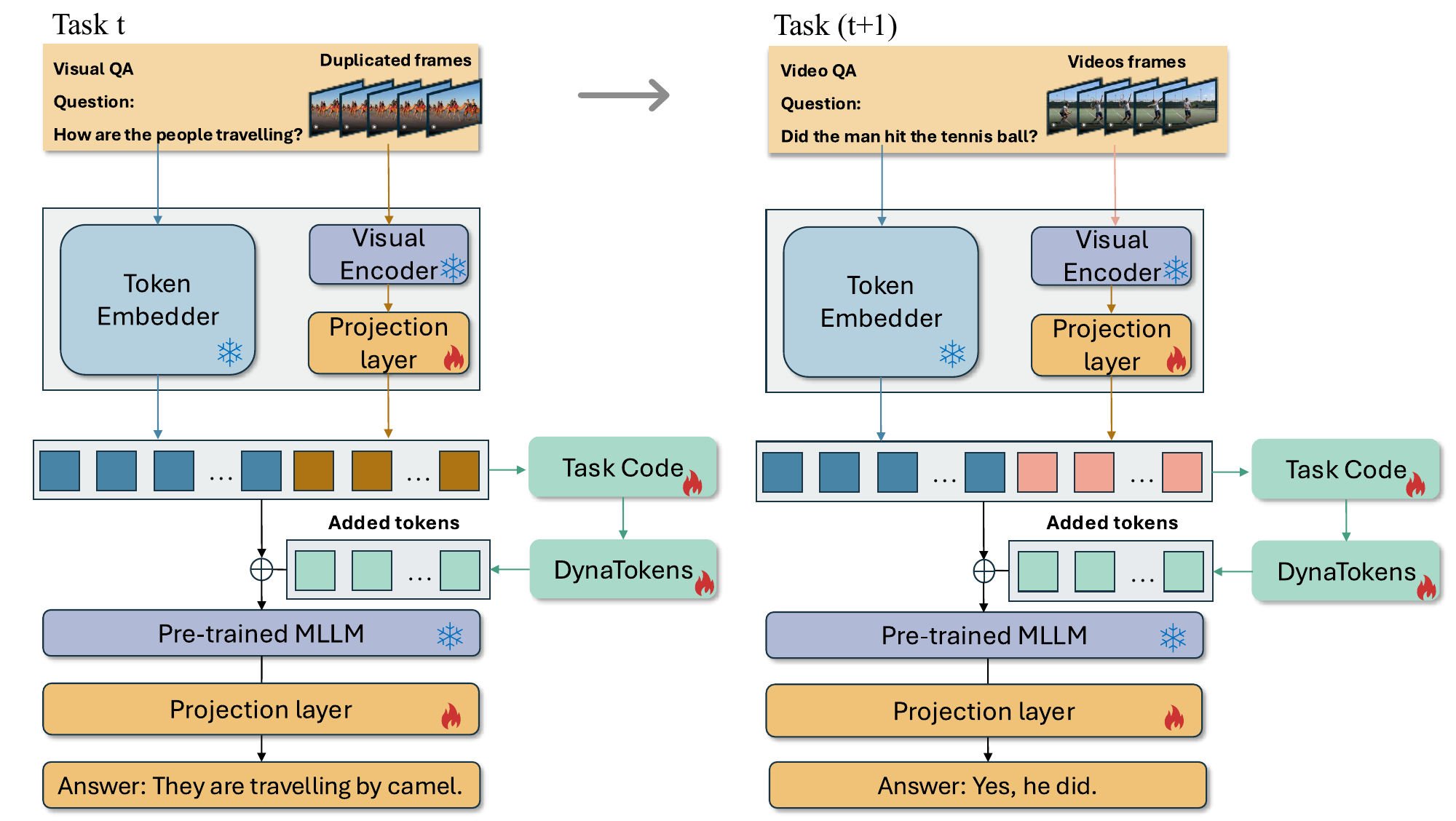} &
    \includegraphics[width=0.20\textwidth]{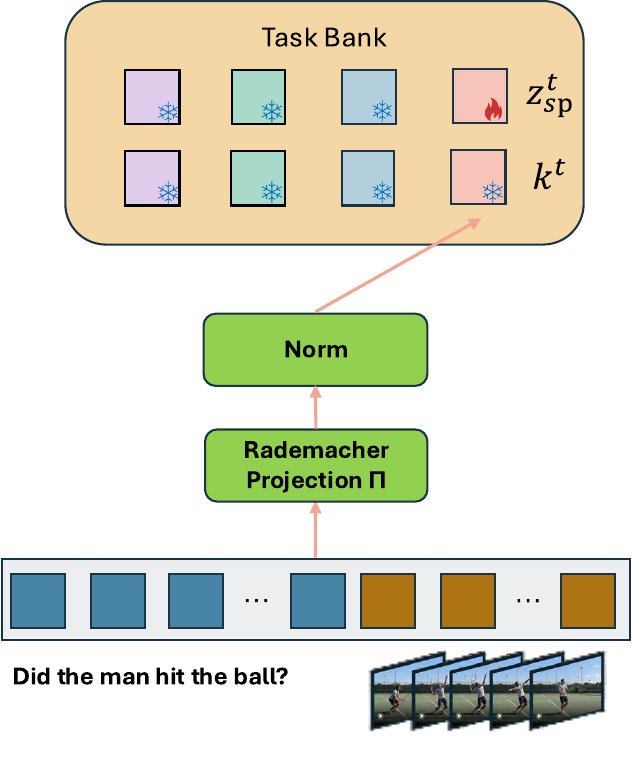} &
    \includegraphics[width=0.20\textwidth]{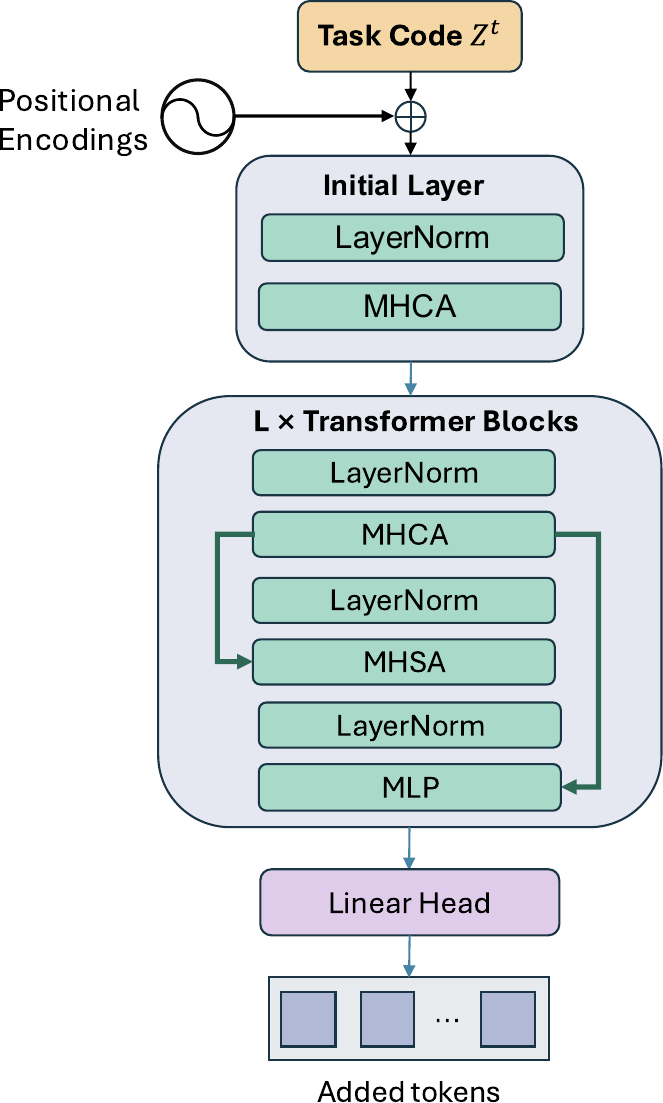} \\
  \end{tabular}
  \caption{\textbf{\model~overview.} (\textbf{Left}) Continual adaptation with \model~ for VideoQA and cross-modal transfer VisualQA$\rightarrow$VideoQA. A \emph{fixed-size} generator synthesises \emph{task-specific} fine-tuning tokens. (\textbf{Middle}) Task bank as key--value memory: multimodal embeddings 
are projected via $\Pi$ and normalised to form retrieval keys $k^t$, 
indexing learnable task codes $z_{sp}^t$ for the token generator. (\textbf{Right}) Transformer-based token-generator architecture.}
  \label{fig:dynatokens_overview}
\end{figure*}

To address these challenges, we draw inspiration from hypernetworks~\cite{ha2017hypernetworks} and propose \textbf{\model}, a transformer-based universal token generator that synthesises adapter prompts \emph{on demand} from compact task codes. Since the generator weights are shared across tasks, its size remains fixed regardless of the task sequence length, so \emph{memory grows only minimally} with the number of tasks. Figure~\ref{fig:dynatokens_overview} illustrates the overall design.

A central challenge in our hypernetwork-based prompt control is twofold: obtaining informative task codes for routing, and updating the generator without forgetting. For routing, we adopt a gradient-free key--value memory~\cite{geva2021transformer} with no learnable parameters. The keys are token-level embeddings from the frozen pretrained backbone, already discriminative enough for reliable retrieval without further training. The values are lightweight task codes, each composed of task-specific and task-invariant learnable modules. Since the routing mechanism introduces no additional trainable parameters, it is \emph{immune to drift by construction}: retention of task knowledge is handled by the generator and its regularisation, not by the router. At inference, retrieved task codes are fed into the shared generator to produce task-adapted prompts on the fly.

To preserve the generator under continual updates, we train it with a meta-learning-inspired objective that first \emph{looks ahead} to avoid sharp, task-specific overfitting directions, and then \emph{looks back} to regularise the update by matching an anchored behaviour on a small set of past task codes. We theoretically connect this objective to sharpness-aware minimisation (SAM)~\cite{foret2021sharpnessaware}, showing that it favours flatter cross-task minima and improves robustness under continual adaptation. Beyond regularisation, the shared generation weights enable \model~ to leverage lightweight auxiliary multimodal objectives effectively, encouraging richer cross-modal token learning.

Empirically, \model~ achieves strong accuracy with substantially reduced forgetting on standard continual VideoQA benchmarks, while also improving zero-shot generalisation and remaining effective on continual dataset with extended domain shift sequences. To further stress-test continual adaptation, we introduce a challenging ImageQA$\rightarrow$VideoQA protocol in which the model must shift from static image understanding to temporal video reasoning, inducing both a \emph{modality mismatch} and a \emph{domain mismatch} between single-frame recognition and temporal reasoning. The strongest recent PEA baseline degrades sharply under this regime, whereas \model~remains robust and outperforms it across both stages, highlighting the effectiveness of token-hypernetwork generation for multimodal continual learning.

\section{Related Work}
\label{sec:Related}

\paragraph{Video Question Answering (VideoQA)}
requires answering questions grounded in video inputs, demanding joint 
reasoning over visual and linguistic context. Existing approaches either 
build specialised video vision--language models~\cite{deitke2025molmo,alayrac2022flamingo,yu2023self} 
or adapt large pretrained LLMs to better capture temporal visual 
context~\cite{yang2022zero,zhang2024simple,zhang2024llamaadapter}. Yet 
they implicitly assume stationary data and are not designed for the 
streaming setting in which VideoQA datasets arrive sequentially with 
substantial distribution shifts.

\paragraph{Continual Learning for VideoQA}
considers a streaming setting where a model is trained sequentially on 
multiple VideoQA datasets, with catastrophic forgetting as the central 
challenge. Replay-based methods rehearse past 
samples~\cite{buzzega2020dark,nguyen2024class} but face prohibitive 
storage costs for raw videos. Regularisation- and architecture-based 
approaches have been explored mainly for multimodal 
classification~\cite{ke2020continual,jha2024clap4clip,liu2025c} and do 
not transfer cleanly to next-token LLM-style VideoQA, motivating recent 
shifts toward parameter-efficient adaptation atop frozen backbones.

\paragraph{Parameter-Efficient Adaptation (PEA)}~\cite{zhang2024llamaadapter,li2021prefix} 
enables continual adaptation by training lightweight task-specific 
modules atop a frozen backbone. We highlight three differentiating 
properties of \model{} against representative baselines:

\textbf{(i) Bank size.} Prompt-pool methods~\cite{wang2022learning,wang2022dualprompt,smith2023coda} 
and per-task prompt-storage methods~\cite{razdaibiedina2023progressive,zeng2025modalprompt} 
maintain prompt collections that grow linearly with the task stream. 
\model{} instead synthesises prompts on demand from compact task codes 
via a \emph{fixed-size} generator, decoupling memory cost from the 
number of tasks.

\textbf{(ii) Routing dependency.} ModalPrompt~\cite{zeng2025modalprompt} 
requires auxiliary CLIP encoders for prompt selection, adding external 
modules and inference cost. \model{} routes directly through the frozen 
backbone's own token embeddings, requiring no external encoders.

\textbf{(iii) Look-ahead regularisation.} The closest baseline, 
\textbf{AdaPrefix++}~\cite{adhikari2025adaprefix++}, also generates 
fine-tuning parameters with a hypernetwork~\cite{ha2017hypernetworks} 
from stored per-task inputs, but regularises against past behaviour 
without look-ahead and offers no theoretical analysis of its regulariser. 
\model{} adds a look-ahead regulariser whose gains grow with more 
look-ahead steps and which admits the sharpness-based characterisation 
in Sec.~\ref{sec:theory}; Appendix~\ref{app:adaprefix} details the 
remaining differences in setting, routing, and architecture. Crucially, existing PEA methods are evaluated primarily in 
unimodal or static image-text settings, leaving continual VideoQA with 
its temporal reasoning demands largely unexplored. We provide an 
extended discussion of prompt-based methods and hypernetwork-based 
continual learning in Appendix~\ref{sec:more_related_work}.

\section{Background}
\label{sec:Background}

\paragraph{Problem Setup.}
We study continual VideoQA over a sequence of tasks $t=1,\dots,T$. Task $t$ provides a dataset $\mathcal{D}_t=\{(V_i^t,Q_i^t,A_i^t)\}_{i=1}^{N_t}$, where $V_i^t\!\in\!\mathbb{R}^{L_v\times C}$, $Q_i^t\!\in\!\mathbb{R}^{L_q\times C}$, and $A_i^t\!\in\!\mathbb{R}^{L_a\times C}$ are the embedded video, question, and answer token sequences, respectively. The goal is to adapt a pretrained multimodal LLM $f_\theta$ to each incoming task while preserving performance on all previous ones. Denoting the causal context $\mathcal{C}_{i,j}^t=(V_i^t,Q_i^t,A_{i,\le j}^t)$, the per-task training objective is the standard next-token negative log-likelihood:
\begin{equation}\label{eq:nll}
  \mathcal{L}_{\mathrm{NLL}}^{t}
  = -\,\mathbb{E}_{i\sim\mathcal{D}_t}\sum_{j=1}^{L_a}
    \log\, p_\theta\!\bigl(a_{i,j}^t\mid\mathcal{C}_{i,j-1}^t\bigr).
\end{equation}

\paragraph{Forgetting.}
Sequentially minimising~\eqref{eq:nll} across tasks induces \emph{catastrophic forgetting}: updates fitted to $\mathcal{D}_t$ may overwrite knowledge required for earlier tasks, especially in the replay-free regime where past data and gradients are unavailable. Parameter-efficient adaptation (PEA) partially mitigates this by freezing the backbone and training only lightweight modules (e.g., adapters, prompts, or LoRA weights), though these modules must still be stored and correctly routed at inference, which becomes increasingly difficult as the task stream grows.

\section{Method}
\label{sec:Method}

\subsection{DynaTokens}
\label{subsec:hypertokens}

We introduce \textbf{\model}, a hypernetwork-based method that synthesises task-specific prompt tokens \emph{on demand} from compact task codes. Each task~$t$ is represented by two lightweight objects: a generative code $z^t$ for prompt synthesis, and a retrieval key $k^t$ for identifying the task at inference. Decoupling \emph{what to generate} from \emph{how to route} enables scalable continual prompt control without storing per-task prompts.

\paragraph{Prompt Generation via a Token Hypernetwork.}
We split the generative code into a shared, task-invariant component and a task-specific one, $z^t = [z_{\mathrm{inv}},\, z_{\mathrm{sp}}^t] \in \mathbb{R}^{N_z \times C_z}$ with $N_z = N_{\mathrm{inv}} + N_{\mathrm{sp}}$. To produce layer-specific prompts, the hypernetwork $H_\phi$ is additionally conditioned on a learnable layer embedding $e_\ell \in \mathbb{R}^{d}$ for each prompted layer $\ell \in \{1, \dots, L_p\}$, where $d$ is its internal width:
\begin{equation}
\label{eq:hypertoken_generator}
\begin{split}
  P_\ell^t &= H_\phi(z^t, e_\ell) \in \mathbb{R}^{N_p \times C}, \\
  P^t &= [P_1^t; \dots; P_{L_p}^t] \in \mathbb{R}^{L_p \times N_p \times C},
\end{split}
\end{equation}
where $N_p$ is the number of prompt tokens per layer and $C$ is the LLM hidden dimension. Layer conditioning lets a single generator specialise prompts to the role of each transformer layer (low-level features vs.\ higher-level semantics) while keeping per-task storage small. For example, on LLaMA2-7B ($C{=}4096$) with $N_{\mathrm{sp}}{=}25$, $C_z{=}256$, $L_p{=}32$, $N_p{=}10$, each task stores only its code $z^t_{\mathrm{sp}}$ and a $D_k{=}256$-dimensional key, $N_{\mathrm{sp}} C_z + D_k \!\approx\! 6.7{\times}10^3$ values (${\approx}26$\,KB), whereas the materialised prompt tensor has $L_p N_p C \!\approx\! 1.3{\times}10^6$ entries, a roughly $200{\times}$ reduction. Each new task therefore adds only kilobytes of storage rather than full prompt blocks.

We instantiate $H_\phi$ as a transformer because prompt synthesis is naturally a structured \emph{set-to-set} mapping: a compact task code must be expanded into many coordinated prompts across layers. The code is projected to the generator width, $X^t = z^t W_z + E_z \in \mathbb{R}^{N_z \times d}$, and the layer identity $e_\ell$ is injected by feature-wise modulation (FiLM) of the internal features. A lightweight pooling head derives $N_p$ initial query tokens $U^{(0)}$ from $X^t$, and $L_h$ pre-norm blocks of cross-attention (MHCA), self-attention (MHSA), and MLP then let the queries repeatedly read from the task code and interact with each other to specialise:
\begin{equation}
\label{eq:h_transformer_block}
\begin{split}
  \bar{U}^{(\ell)} &= U^{(\ell)} + \mathrm{MHCA}\bigl(\mathrm{LN}(U^{(\ell)}), X^t\bigr), \\
  \widetilde{U}^{(\ell)} &= \bar{U}^{(\ell)} + \mathrm{MHSA}\bigl(\mathrm{LN}(\bar{U}^{(\ell)})\bigr), \\
  U^{(\ell+1)} &= \widetilde{U}^{(\ell)} + \mathrm{MLP}\bigl(\mathrm{LN}(\widetilde{U}^{(\ell)})\bigr).
\end{split}
\end{equation}
A shared head $W_o \in \mathbb{R}^{d \times C}$ maps the output states to the LLM width, $P_\ell^t = U^{(L_h)} W_o$; the $L_p$ layers are generated in one batched call and stacked into $P^t$. Full architectural details are given in Appendix~\ref{app:details}.

\paragraph{Decoupled Retrieval Keys.}
Codes optimised for generation need not be good retrieval prototypes, so we maintain a separate key $k^t \in \mathbb{R}^{D_k}$ per task, grounded directly in the frozen input embeddings $(V_i^t, Q_i^t)$. For each sample we form
\begin{equation}
\label{eq:sample_key}
  r_i^t = \mathrm{Norm}\!\bigl(\Pi\,\mathrm{Pool}([V_i^t;\, Q_i^t])\bigr),
\end{equation}
where $\Pi \in \mathbb{R}^{D_k \times C}$ is a fixed Rademacher random projection that approximately preserves distances under the Johnson--Lindenstrauss lemma~\cite{johnson1984extensions}. The task key is its normalised mean,
\begin{equation}
\label{eq:task_key}
  k^t = \mathrm{Norm}\!\Bigl(\tfrac{1}{N_t}\sum\nolimits_{i=1}^{N_t} r_i^t\Bigr).
\end{equation}
Since $\Pi$ is frozen and the underlying embeddings come from frozen modules, all task keys live in a stable, drift-free space. Routing thus adds \emph{no trainable parameters} and requires no rehearsal or anti-forgetting on the router itself, sidestepping a major failure mode of learned routers in continual learning. In practice, $k^t$ is accumulated with a moving average over minibatches, and pooling $Q_i^t$ alone already yields discriminative keys.

\paragraph{Task-Code Initialisation.}
We set $N_{\mathrm{inv}}\!=\!\tfrac{1}{5}N_z$, reserving a small portion of the code budget for shared structure and the rest for task-specific capacity. $z_{\mathrm{inv}}$ is initialised once with $\mathcal{N}(0, \sigma^2 I)$, $\sigma\!=\!0.02$, and updated across tasks. The first task's code is sampled the same way; for $t > 1$, we warm-start $z_{\mathrm{sp}}^t$ from previously learned codes weighted by key similarity:
\begin{equation}
\label{eq:z_sp_init}
  z_{\mathrm{sp}}^t = \sum_{\tau < t} w_\tau^t \, z_{\mathrm{sp}}^\tau + \epsilon, \quad \epsilon \sim \mathcal{N}(0,\sigma^2 I),
\end{equation}
\begin{equation}
\label{eq:z_sp_weights}
  w_\tau^t = \frac{\exp(\cos(k^t, k^\tau)/\rho_z)}{\sum_{j<t}\exp(\cos(k^t, k^j)/\rho_z)}.
\end{equation}
This is a soft retrieval in the key space defined above: semantically related past tasks contribute more, providing positive transfer rather than optimising from scratch. With a small temperature ($\rho_z\!=\!0.1$), the weighting concentrates on the most similar past tasks, approximating soft nearest-neighbour retrieval~\cite{oord2018representation}.

\paragraph{Training Objective and LA-Reg.}
The base objective is the autoregressive negative log-likelihood conditioned on the generated prompts:
\begin{equation}
\label{eq:nll_sum_ht}
  \mathcal{L}_{\mathrm{NLL}}^t = -\,\mathbb{E}_i \sum_{j=0}^{N_{i,a}^t-1} \log p\!\bigl(a_{i,j+1}^t \mid \mathcal{C}_{i,j}^t\bigr),
\end{equation}
where $\mathcal{C}_{i,j}^t = (V_i^t, Q_i^t, A_{i,\le j}^t, P^t)$. Since $H_\phi$ is shared across tasks, na\"ively updating it on $\mathcal{D}_t$ can drift the prompts produced for previous task codes. We prevent this with a \emph{look-ahead regularisation} (\textbf{LA-Reg}) that penalises such drift on stored codes. Let $\phi^\star$ be the generator snapshot before task~$t$ and $\Delta\phi$ the update from $M$ inner-loop steps on $\mathcal{L}_{\mathrm{NLL}}^t$:
\begin{equation}
\label{eq:la_reg_h}
  \mathcal{L}_{\mathrm{LA\text{-}Reg}}^t = \sum_{\tau < t} \bigl\| H_{\phi^\star}(z^\tau) - H_{\phi+\Delta\phi}(z^\tau) \bigr\|_2^2.
\end{equation}
LA-Reg requires only the stored \emph{codes}, not raw data: previous prompt behaviour is reconstructed on the fly through $H_{\phi^\star}$. We additionally anchor the shared code with a quadratic penalty $\|z_{\mathrm{inv}} - z_{\mathrm{inv}}^\star\|_2^2$, where $z_{\mathrm{inv}}^\star$ is its pre-task snapshot; in practice this term remains small throughout training and acts as a mild, non-restrictive prior on the shared capacity.

\begin{figure}[t]
  \centering
  \captionsetup{font=small}
  \includegraphics[width=0.9\columnwidth]{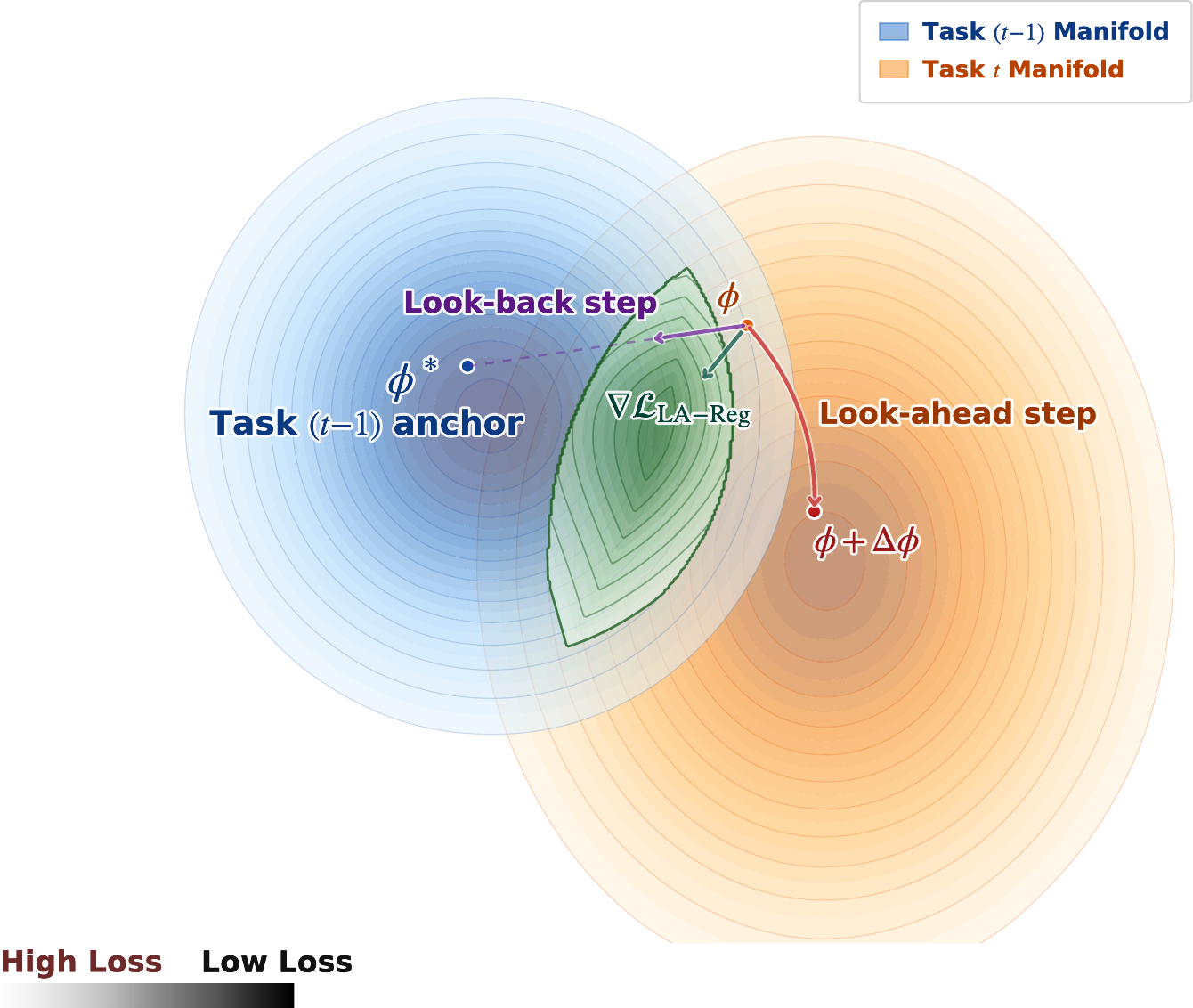}
    \caption{\textbf{Geometry of LA-Reg.} LA-Reg balances progress along the task-$t$ gradient with alignment to the task-$(t\!-\!1)$ anchor, steering optimisation into the shared low-loss basin (green) and yielding flatter cross-task minima. Regularisation acts in prompt-output space, not parameter space.}
  \label{fig:la_reg_geometry}
\end{figure}
\paragraph{Auxiliary Multimodal Supervision.}
Following recent VideoQA work~\cite{tan2025bisecle, cai2024empowering}, we augment $\mathcal{L}_{\mathrm{NLL}}^t$ with two cross-modal auxiliary objectives, $\mathcal{L}_{\mathrm{VAQ}}^t$ and $\mathcal{L}_{\mathrm{QAV}}^t$: the next-token NLL of $Q$ given $(V, A, P)$ and of $V$ given $(Q, A, P)$, respectively. A \emph{causal} view~\cite{pearl2009causality, nguyen2023causal} explains their relative weighting. Under the VideoQA graph $V\!\rightarrow\!Q$, $V\!\rightarrow\!A$, $Q\!\rightarrow\!A$, the conditional $p(Q\mid V, A)$ follows a feasible causal direction and yields an identifiable signal; $p(V\mid Q, A)$, in contrast, is \emph{anti-causal} and underdetermined, as many videos can map to the same $(Q, A)$. Yet the anti-causal loss is not discarded: even as a poor generative target, it still pushes the prompts and the QA-side representation to stay \emph{visually grounded}. We therefore keep $\mathcal{L}_{\mathrm{QAV}}^t$ in the same next-token form but give it a small weight, an asymmetry that ablations in Sec.~\ref{sec:more_ablation_studies} confirm: removing it weakens token learning, while up-weighting it hurts performance.

\paragraph{Full Objective.}
The task-$t$ objective combines the base, regularisation, and auxiliary terms:
\begin{equation}
\label{eq:overall_loss}
\begin{split}
  \mathcal{L}_{\mathrm{Final}}^t
  &= \mathcal{L}_{\mathrm{NLL}}^t
  + \lambda_{\mathrm{LA\text{-}Reg}}\,\mathcal{L}_{\mathrm{LA\text{-}Reg}}^t \\
  &\quad + \lambda_{\mathrm{inv}}\, \|z_{\mathrm{inv}} - z_{\mathrm{inv}}^\star\|_2^2 \\
  &\quad + \lambda_{\mathrm{VAQ}}\,\mathcal{L}_{\mathrm{VAQ}}^t
  + \lambda_{\mathrm{QAV}}\,\mathcal{L}_{\mathrm{QAV}}^t.
\end{split}
\end{equation}
Unless stated otherwise, we fix $\lambda_{\mathrm{LA\text{-}Reg}}\!=\!0.5$, $\lambda_{\mathrm{inv}}\!=\!0.1$, $\lambda_{\mathrm{VAQ}}\!=\!1.0$, and $\lambda_{\mathrm{QAV}}\!=\!0.1$ across all benchmarks, following the causal weighting above. The small $\lambda_{\mathrm{inv}}$ reflects that its underlying penalty stays naturally small across tasks, yet ablations (Sec.~\ref{sec:Experiments}) show it still contributes meaningfully to retention. This single default leaves \model~tuning-free at deployment. Since $\mathcal{L}_{\mathrm{LA\text{-}Reg}}^t$ sums over the $t-1$ stored codes, its magnitude grows along the stream under a fixed weight; averaging over stored codes, equivalently scaling $\lambda_{\mathrm{LA\text{-}Reg}} \propto \tfrac{1}{t-1}$, keeps the per-task strength constant. Our fixed setting nevertheless remains stable over the 13-task stream.

\paragraph{Task Bank and Unknown-Task Inference.}
After task~$t$, we cache only $(k^t, z_{\mathrm{sp}}^t)$ in a compact bank $\mathcal{B}_t = \mathcal{B}_{t-1}\cup\{(k^t, z_{\mathrm{sp}}^t)\}$. At test time, given $(V_{\mathrm{test}}, Q_{\mathrm{test}})$, we compute $r_{\mathrm{test}}$ via Eq.~\eqref{eq:sample_key} and retrieve the nearest task,
\begin{equation}
\label{eq:task_retrieval}
  t^\star = \arg\max_{\tau \le t} \cos(r_{\mathrm{test}}, k^\tau),
\end{equation}
reconstruct $z^{t^\star} = [z_{\mathrm{inv}}, z_{\mathrm{sp}}^{t^\star}]$, and synthesise $P_{\mathrm{test}} = H_\phi(z^{t^\star})$ on the fly. Inference requires \emph{no task identifier}, extending \model~to the more challenging \emph{task-agnostic} continual learning setting where task identity is hidden at test time. Routing accuracy, when nearest-key retrieval errs, and its bounded downstream impact are analysed in Appendix~\ref{app:routing_acc}.

\subsection{Theoretical Analysis}
\label{sec:theory}
We give two complementary views on why \model~resists forgetting: (i) LA-Reg acts as a task-wise sharpness penalty, and (ii) the shared generator confines adaptation to a low-dimensional prompt manifold. Full proofs are in Appendix~\ref{sec:appendix_theoretical_results}.

\paragraph{LA-Reg as a Sharpness Penalty.}
Let $J_\phi(z) \coloneqq \partial H_\phi(z) / \partial \phi$ and $g(\phi) \coloneqq \nabla_\phi \mathcal{L}_{\mathrm{NLL}}^t(\phi)$, so the look-ahead update in Eq.~\ref{eq:la_reg_h} is $\Delta\phi = -\eta\, g(\phi)$.
\begin{theorem}[Task-wise sharpness penalty]
\label{thm:la_reg_sam}
Assume $H_\phi$ is twice differentiable with bounded Jacobian and remainder, $\|H_{\phi^\star}(z^\tau) - H_\phi(z^\tau)\|_2 \le \varepsilon$, and $\|g(\phi)\|_2 \le G$. Then for small $\eta$, there exists $C>0$ such that
\begin{equation}
\label{eq:la_sam_expansion}
\begin{split}
\bigl| \mathcal{L}_{\mathrm{LA\text{-}Reg}}^{t}(\phi) &- \eta^2 \!\sum_{\tau<t}\! \| J_{\phi}(z^\tau)\, g(\phi) \|_2^2 \bigr| \\
&\le C(\varepsilon^2 + \varepsilon\eta + \eta^3).
\end{split}
\end{equation}
\end{theorem}
The leading term $\eta^2 \sum_{\tau<t} \|J_\phi(z^\tau) g(\phi)\|_2^2$ flags ``sharp'' directions where fitting task $t$ erodes past behaviour. LA-Reg penalises them (Fig.~\ref{fig:la_reg_geometry}), encouraging flatter cross-task basins. Theorem~\ref{thm:la_reg_sam} is a \emph{first-order characterisation} of what LA-Reg penalises, a cross-task sharpness term on the generator parameters $\phi$, rather than an anti-forgetting bound. The connection to sharpness-aware minimisation~\cite{foret2021sharpnessaware} is correspondingly scoped: both penalise sensitivity to a gradient-aligned parameter perturbation, but SAM perturbs all model parameters with respect to the training loss, whereas LA-Reg perturbs only $\phi$ and measures the induced drift of past-task prompt outputs.

\paragraph{Low-Rank Adaptation on a Shared Manifold.}
Local updates to a generated prompt $H_\phi(z^t)$ lie in the range of the Jacobian $J_z H_\phi(z^t)$, whose dimension is bounded by $N_z C_z$, the task-code budget. Since $N_z C_z \!\ll\! L_p N_p C$, \model~implicitly restricts adaptation to a low-dimensional shared manifold inside the much larger prompt space, with task codes providing task-specific control and the shared generator biasing prompts toward reusable cross-task structure. A formal statement and proof are deferred to Appendix~\ref{sec:appendix_manifold}.

\section{Experiments}
\label{sec:Experiments}

% Add error bars

% Analysising the impact of task order

% We briefly summarise the main details of our three experimental setups; further details are provided in Appendix~\ref{sec:more_exp_details}.

% {47.91} & \best{4.60}

% \rowcolor{black!3}\textbf{\model~ \venue{ours}} 
% & {56.99} & {5.18} & {68.94} & {9.50} & \textbf{--} & \textbf{--}
% & \textemdash & \textemdash & \textemdash & \textemdash
% & \textemdash & \textemdash & \textemdash & \textemdash & \textemdash & \textemdash \\
% \rowcolor{black!3}\textbf{\model~ + vaq loss \venue{ours}} 
% & \runner{62.75} & \runner{3.44} & \textbf{--} & \textbf{--} & \textbf{--} & \textbf{--}
% & \textemdash & \textemdash & \textemdash & \textemdash
% & \textemdash & \textemdash & \textemdash & \textemdash & \textemdash & \textemdash \\
% \rowcolor{black!3}\textbf{\model~ + MeLA-Reg \venue{ours}} 
% & {58.88} & {4.00} & \textbf{--} & \textbf{--} & \textbf{--} & \textbf{--}
% & \textemdash & \textemdash & \textemdash & \textemdash
% & \textemdash & \textemdash & \textemdash & \textemdash & \textemdash & \textemdash \\
% \rowcolor{black!3}\textbf{\model~ + MeLA-Reg + vaq loss \venue{ours}} 
% & \best{64.40} & \best{3.13} & \textbf{--} & \textbf{--} & \textbf{--} & \textbf{--}
% & \textemdash & \textemdash & \textemdash & \textemdash
% & \textemdash & \textemdash & \textemdash & \textemdash & \textemdash & \textemdash \\

\subsection{Continual VideoQA}

\subsubsection{Setup}

\paragraph{Datasets and Settings.}
We evaluate on two continual VideoQA benchmarks. Following~\cite{tan2025bisecle}, NExT-QA~\cite{xiao2021next} is split into eight tasks by question type (\texttt{TP}, \texttt{CW}, \texttt{DC}, \texttt{TC}, \texttt{DL}, \texttt{DO}, \texttt{TN}, \texttt{CH}), spanning causal, descriptive, and temporal categories; following~\cite{cai2024empowering}, DramaQA~\cite{choi2021dramaqa} is split in the order What\,$\to$\,Who\,$\to$\,Where\,$\to$\,How\,$\to$\,Why. We further introduce a \emph{continual-dataset} setting that chains the two benchmarks as a longer \emph{domain-incremental} stream, and report \emph{zero-shot performance} on unseen datasets to probe robustness under test-time distribution shift.

\paragraph{Baselines.}
We compare against representative PEA methods for LLM-based continual VideoQA: Bisecle~\cite{tan2025bisecle}, DAM~\cite{cheng2025dam}, ColPro~\cite{cai2024empowering}, ProgPrompt~\cite{razdaibiedina2023progressive}, LAE~\cite{gao2023unified}, DualPrompt~\cite{wang2022dualprompt}, and L2P~\cite{wang2022learning}, plus a na\"ive LLaMA-Adapter~\cite{zhang2024llamaadapter} reference.

\paragraph{Metrics.}
We report two standard metrics: average accuracy \ACC{} (mean top-1 across tasks after the sequence) and average forgetting \FOG{} (mean drop from each task's peak to its final accuracy).

\paragraph{Implementation.}
All methods share a frozen backbone: LLaMA-2-7B~\cite{touvron2023llama} as the LLM and CLIP ViT-L/14~\cite{radford2021learning, dosovitskiy2021an} as the visual encoder. Unless stated otherwise, we use LLaMA-Adapter~\cite{zhang2024llamaadapter} with 32 adapter layers. Following~\cite{tan2025bisecle, cai2024empowering}, each video is sampled to 10 frames at $224\!\times\!224$; text is padded/truncated to 128 tokens on NExT-QA and 280 on DramaQA. We train five epochs on NExT-QA and six on DramaQA, with two look-ahead steps for LA-Reg. Full details are in Appendix~\ref{sec:additional_exp_details}.

\begin{table}[!t]
\centering
\caption{\emph{Continual VideoQA} results on 
\NextQAcap~and \DramaQcap~with per-dataset accuracy and forgetting. 
\bestcap{Bold} and \runnercap{underline} denote the best and second-best results.}
\label{tab:videoqa_results_final}

\begingroup
% Make \venue smaller ONLY inside this table
\renewcommand{\venue}[1]{\textcolor{black!60}{\textsf{(\fontsize{6}{6.8}\selectfont #1)}}}

\fontsize{7}{8}\selectfont
\setlength{\tabcolsep}{2.5pt}
\renewcommand{\arraystretch}{0.9}

\resizebox{0.5\textwidth}{!}{%
\begin{tabular}{L{3.0cm}*{4}{c}}
\toprule
\textbf{Method ~\venue{Venue}} &
\hdr{NextQA}{NExT\mbox{-}QA} &
\hdr{DramaQ}{DramaQA} \\
& \HACC & \HFOG & \HACC & \HFOG \\
\midrule
\rowcolor{black!3}{LLaMA-Adapter ~\venue{ICLR'24}} 
& 46.58 & 13.83 & 60.99 & 24.39 \\
\midrule
L2P ~\venue{CVPR'22}        
& 48.82 & 12.25 & 62.50 & 20.67 \\
\rowcolor{black!3}DualPrompt ~\venue{ECCV'22} 
& 50.62 & 11.74 & 65.89 & 17.93 \\
LAE ~\venue{ICCV'23}         
& 49.38 & 11.47 & 65.82 & 17.35 \\
\rowcolor{black!3}ProgPrompt ~\venue{ICLR'23}  
& 53.95 & 10.69 & 67.92 & 14.95 \\
ColPro ~\venue{EMNLP'24}     
& {55.14} & {7.43} & {71.24} & {12.64} \\
\rowcolor{black!3}DAM ~\venue{WACV'25}        
& 53.88 &  9.99 & 67.37 & 15.19 \\
Bisecle ~\venue{NeurIPS'25}     
& \runner{62.37} & \runner{5.34} & \runner{71.49} & \runner{10.37} \\
\midrule
\rowcolor{black!3}\textbf{\model~\venue{ours}} 
& \best{64.11} & \best{4.76} & \best{72.52} & \best{9.76} \\
\bottomrule
\end{tabular}
}% end resizebox

\endgroup
\end{table}
\subsubsection{Results}

\paragraph{Performance Comparison.}
Table~\ref{tab:videoqa_results_final} shows that \model~sets a new state of the art on both benchmarks, achieving the highest \ACC{} and lowest \FOG{} across the board. On NExT-QA, \model~lifts \ACC{} by $+1.74$ over the previous best (Bisecle) while cutting \FOG{} by $0.58$, a non-trivial gain when both metrics are averaged over eight sequential tasks. The advantage carries over to DramaQA, a domain with markedly different question distributions, where \model~surpasses Bisecle by $+1.03$ \ACC{} and $-0.61$ \FOG{} and outperforms strong prompt-based baselines such as ProgPrompt and ColPro by even larger margins. These consistent gains underscore the effectiveness of controlled prompt synthesis via the shared token hypernetwork.

\paragraph{\ACC~and \FOG~Changes.}
Fig.~\ref{fig:avg-acc-for} tracks \ACC{} and \FOG{} after each task on NExT-QA. \model~consistently leads Bisecle throughout the sequence, with the gap widening as tasks accumulate. The largest single-task gain reaches $+5.07$ \ACC{} (DC) and $-2.75$ \FOG{} (TN), confirming that \model~not only learns new tasks better but also forgets less along the way.

\paragraph{Continual-dataset Setting.}
We further evaluate a more demanding setting that chains NExT-QA and DramaQA into a single $13$-task stream, introducing both a \emph{longer sequence} and a substantial \emph{domain shift} (real-world $\rightarrow$ movie videos). Table~\ref{tab:continual_dataset} reports \ACC{} and \FOG{} averaged over all $13$ tasks: \model~outperforms all baselines by a clear margin, confirming its robustness under extended task streams and cross-domain transitions.

\begin{table}[!t]
\centering
\caption{\emph{Continual-dataset} results on the joint NExT-QA $\rightarrow$ DramaQA stream ($13$ tasks total). We report \ACC{} and \FOG{} averaged across all tasks. \bestcap{Bold} and \runnercap{underline} denote the best and second-best results.}
\label{tab:continual_dataset}
\begingroup
\renewcommand{\venue}[1]{\textcolor{black!60}{\textsf{(\fontsize{6}{6.8}\selectfont #1)}}}
\fontsize{7}{8}\selectfont
\setlength{\tabcolsep}{8.0pt}
\renewcommand{\arraystretch}{1.2}
\begin{tabular}{L{3.0cm}cc}
\toprule
\textbf{Method ~\venue{Venue}} & \HACC & \HFOG \\
\midrule
\rowcolor{black!3}ProgPrompt~\venue{ICLR'23}   & 48.43 & \runner{11.88} \\
ColPro~\venue{EMNLP'24}                        & 54.18 & 12.49 \\
\rowcolor{black!3}Bisecle~\venue{NeurIPS'25}   & \runner{55.31} & {13.30} \\
\midrule
\textbf{\model~\venue{ours}}                   & \best{61.96} & \best{11.72} \\
\bottomrule
\end{tabular}
\endgroup
\end{table}

\paragraph{Zero-shot Inference}
We probe robustness on \emph{unseen} datasets at test time, which also tests our routing mechanism under out-of-distribution inputs. Each method's checkpoint trained on one benchmark is evaluated zero-shot on the other (NExT-QA $\leftrightarrow$ DramaQA). From Table~\ref{tab:main_zeroshot}, \model~generalises most robustly against baselines.

\begin{table}[!t]
\centering
\caption{\emph{Zero-shot inference} results across NExT-QA and DramaQA: train on one benchmark, evaluate on the other. We report top-1 accuracy. \bestcap{Bold} and \runnercap{underline} denote the best and second-best results.}
\label{tab:main_zeroshot}
\begingroup
\renewcommand{\venue}[1]{\textcolor{black!60}{\textsf{(\fontsize{6}{6.8}\selectfont #1)}}}
\fontsize{7}{8}\selectfont
\setlength{\tabcolsep}{3.0pt}
\renewcommand{\arraystretch}{1.2}
\begin{tabular}{L{3.0cm}cc}
\toprule
\textbf{Method ~\venue{Venue}} &
\textbf{OOD on NExT-QA} &
\textbf{OOD on DramaQA} \\
\midrule
\rowcolor{black!3}ProgPrompt ~\venue{ICLR'23}
& \runner{30.55} & {39.24} \\
ColPro ~\venue{EMNLP'24}
& 27.73 & \runner{51.47} \\
\rowcolor{black!3}Bisecle ~\venue{NeurIPS'25}
& 27.86 & 49.07 \\
\midrule
\textbf{\model~\venue{ours}}
& \best{36.29} & \best{52.83} \\
\bottomrule
\end{tabular}
\endgroup
\end{table}

\subsubsection{Analysis and Ablation Studies}

% \paragraph{HyperToken (S/M/L): Scaling PEA continual learning} We analyse performance of \model~ across different size of networks.

\begin{table}[t]
\centering
\caption{\emph{Contribution of loss terms} on NExT-QA. $\mathcal{L}_{\mathrm{Inv}}$ denotes invariance regularisation and $\mathcal{L}_{\mathrm{Aux}}$ denotes the weighted sum of auxiliary objectives $\mathcal{L}_{\mathrm{VAQ}}$ and $\mathcal{L}_{\mathrm{QAV}}$.}
\label{tab:ablate_losses_ticks}

\begingroup
\fontsize{7}{8}\selectfont
\setlength{\tabcolsep}{8.0pt}
\renewcommand{\arraystretch}{1.1}

\begin{tabular}{ccc!{\vrule width 0.95pt}cc}
\toprule
$\mathcal{L}_{\text{LA-Reg}}$ &
$\mathcal{L}_{\text{Inv}}$ &
$\mathcal{L}_{\text{Aux}}$ &
\HACC & \HFOG \\
\midrule

\rowcolor{black!3}
\xmark & \xmark & \xmark & 21.28 & 16.46 \\
\xmark & \xmark & \cmark & 54.54 & 8.83 \\

\rowcolor{black!3}
\cmark & \xmark & \xmark & 62.12 & 5.46 \\
\midrule
\cmark & \cmark & \xmark & 62.41 & 5.19 \\

\rowcolor{black!3}
\cmark & \xmark & \cmark & \runner{63.79} & \runner{4.88} \\
\midrule
\cmark & \cmark & \cmark & \best{64.11} & \best{4.76} \\

\bottomrule
\end{tabular}
\endgroup

\end{table}

\paragraph{Loss contribution.}
Table~\ref{tab:ablate_losses_ticks} ablates the three training terms of \model~on NExT-QA. The largest gain comes from $\mathcal{L}_{\mathrm{LA}}$: adding it alone (row 3) lifts \ACC{} from $21.28$ to $62.12$ and cuts \FOG{} by over $3\times$, confirming look-ahead regularisation as the principal driver of anti-forgetting. The auxiliary objective $\mathcal{L}_{\mathrm{Aux}}$ adds a complementary $+1.67$ \ACC{} and $-0.58$ \FOG{}, pushing the model into the state-of-the-art regime by enriching cross-modal grounding. The invariance term $\mathcal{L}_{\mathrm{Inv}}$ contributes a smaller but consistent improvement that stabilises the shared code. The full objective is best on both metrics, indicating the three components are complementary rather than redundant.

% On NeXT-QA, loss analysed: L_ctr, LA-Reg and $\omega$-Reg, L_Ques, L_Vid, L_Tok

\paragraph{Effect of look-ahead steps.}
Table~\ref{tab:ablation_look_ahead} varies look-ahead steps in Eq.~\ref{eq:la_reg_h} from $0$ to $2$. Both \ACC{} and \FOG{} improve monotonically with more steps ($54.54 \to 62.30 \to 63.92 \to 64.11$ \ACC{}), confirming that look-ahead updates strengthen anti-forgetting by anticipating current-task drift on past behaviour.

\begin{table}[t]
\centering\small
\setlength{\tabcolsep}{3.0pt}
\begin{tabular}{lccc}
\toprule
 & Bisecle & ProgPrompt & \model{} \\
\midrule
Storage / task & $\approx$16\,KB & $\approx$5\,MB & $\approx$26\,KB \\
Latency (ms/query) & 47.7 & 48.4 & 49.3 \\
\bottomrule
\end{tabular}
\caption{\textbf{Deployment cost.} Per-task storage and inference latency (batch $=1$). Training cost is analysed in Appendix~\ref{app:efficiency}.}
\label{tab:deploy_cost}
\end{table}

\paragraph{Compute, memory, and inference cost.}
Table~\ref{tab:deploy_cost} reports the two costs that recur at deployment. Per-task storage is a compact code of ${\sim}26$\,KB and stays \emph{constant} along the stream, whereas ProgPrompt materialises a prompt per task and its memory grows with the sequence ($21.16\!\to\!25.43$\,GB on NExT-QA; Table~\ref{tab:compute_cost}). Inference latency is essentially identical across methods (within ${\sim}3\%$), since prompt synthesis is a single forward pass through the small generator. Training cost is higher for \model{}, owing to the look-ahead inner loop, but this overhead is confined to training and does not recur at deployment (Appendix~\ref{app:efficiency}); moreover, the $1$-step variant halves training time ($5.55\!\to\!2.46$\,h), trims peak memory by ${\sim}20\%$, and still surpasses Bisecle ($63.92$ vs $62.37$ \ACC{}, $5.02$ vs $5.34$ \FOG{}), making it a practical default.

\paragraph{Token Analysis.}
Figure~\ref{fig:token_corr} (middle/right) visualises task-specific prompt tokens at the silhouette-optimal mid layers ($i = 12$ on NExT-QA, $i = 9$ on DramaQA). Across both datasets, the layout suggests a \emph{skill-oriented organisation} that does not strictly follow the dataset-defined taxonomies: tasks demanding similar reasoning mechanisms tend to cluster together (e.g., DL/TC for spatial-temporal grounding on NExT-QA; DO/DL for visual grounding on DramaQA), while semantically distinct tasks (CW, TP) appear as relative outliers. The clusters remain reasonably compact and well separated, consistent with \model{} retaining discriminative task representations across the continual training trajectory and providing qualitative evidence of mitigated catastrophic forgetting. A more detailed cluster-level analysis is provided in Appendix~\ref{app:token_analysis}.

%Interestingly, same group form closing to each other. 

% Choose layer 16, 24, 32, plot token distribution of every task at the end of training. 
% Store checkpoints, task code at the end is ok!

% Preamble (if not already):
% \usepackage{booktabs}
% \usepackage{graphicx}
% \usepackage{xcolor,colortbl}

% Preamble:
% \usepackage{caption} % for \captionof
% \usepackage{booktabs}
% \usepackage{xcolor,colortbl}
% \usepackage{graphicx}

\begin{figure*}[t]
\centering
\setlength{\tabcolsep}{6pt}
\renewcommand{\arraystretch}{1.05}
\fontsize{7}{8}\selectfont

% --- LEFT: two plots ---
\begin{minipage}[t]{0.81\textwidth}
  \vspace{0pt}
  \centering
  \begin{minipage}[t]{0.49\linewidth}
    \vspace{0pt}
    \centering
    \includegraphics[width=\linewidth]{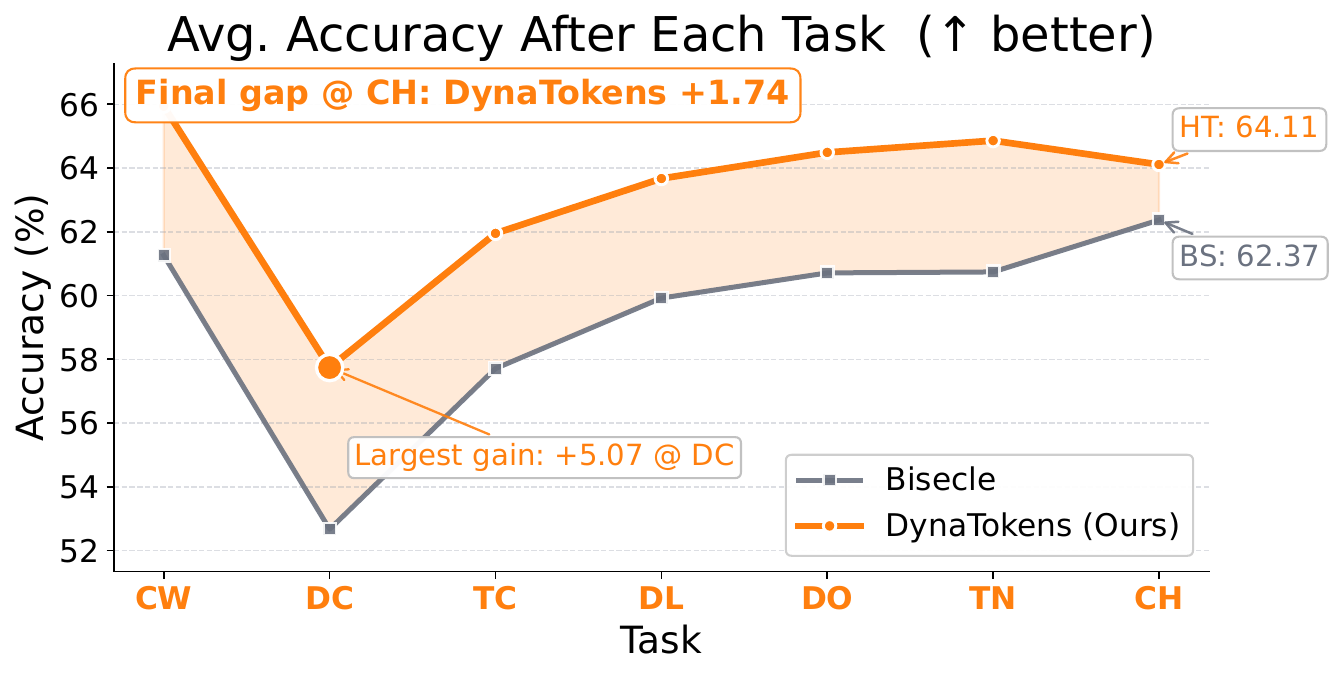}
  \end{minipage}\hfill
  \begin{minipage}[t]{0.49\linewidth}
    \vspace{0pt}
    \centering
    \includegraphics[width=\linewidth]{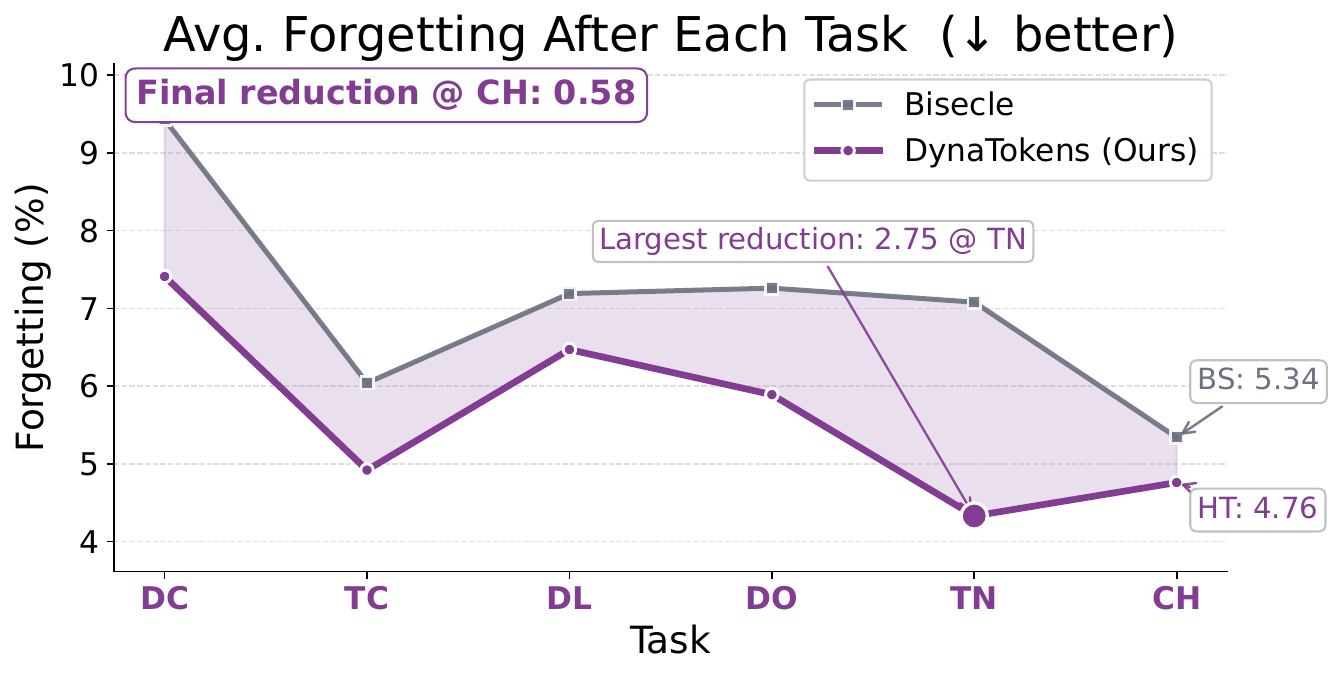}
  \end{minipage}

  \vspace{0.35em}
    \captionof{figure}{\model~consistently surpasses Bisecle in average accuracy and forgetting across tasks. TP is excluded from the accuracy plot (uniformly low at start) and TP--CW from the forgetting plot (forgetting starts at DC).}
  \label{fig:avg-acc-for}
\end{minipage}\hfill
% --- RIGHT: table ---
\begin{minipage}[t]{0.18\textwidth}
  \vspace{0pt}
  \vspace{8mm}
  \centering
\begin{tabular}{c cc}
  \toprule
  \textbf{Steps} & \textbf{Acc.} $\uparrow$ & \textbf{For.} $\downarrow$ \\
  \midrule
  \rowcolor{black!3} \xmark & 54.54  & 8.83 \\
  0 & 62.30 & 5.40 \\
  \rowcolor{black!3} 1 & \runner{63.92} & \runner{5.02} \\
  2 & \best{64.11} & \best{4.76} \\
  \bottomrule
\end{tabular}

  \captionof{table}{\emph{Effect of look-ahead steps} on NExT-QA.}
  \label{tab:ablation_look_ahead}
\end{minipage}

\end{figure*}

\begin{figure*}[t]
  \centering
  \includegraphics[width=0.47\textwidth]{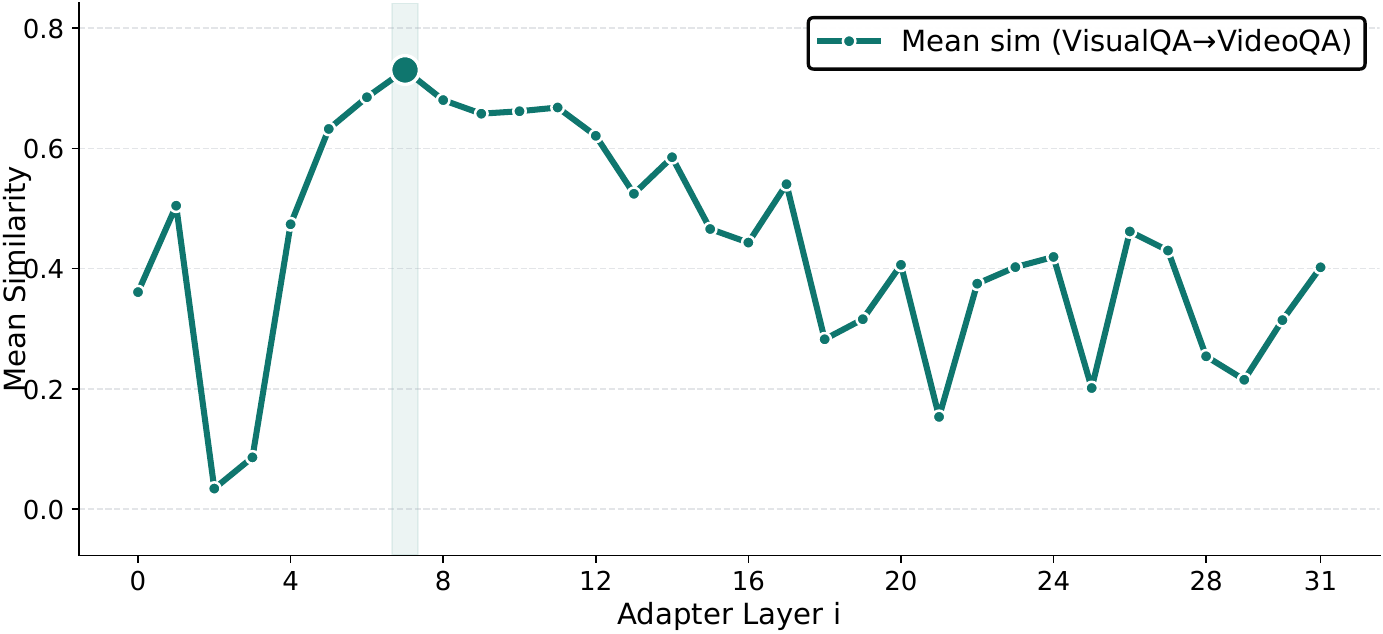}\hfill
  \includegraphics[width=0.26\textwidth]{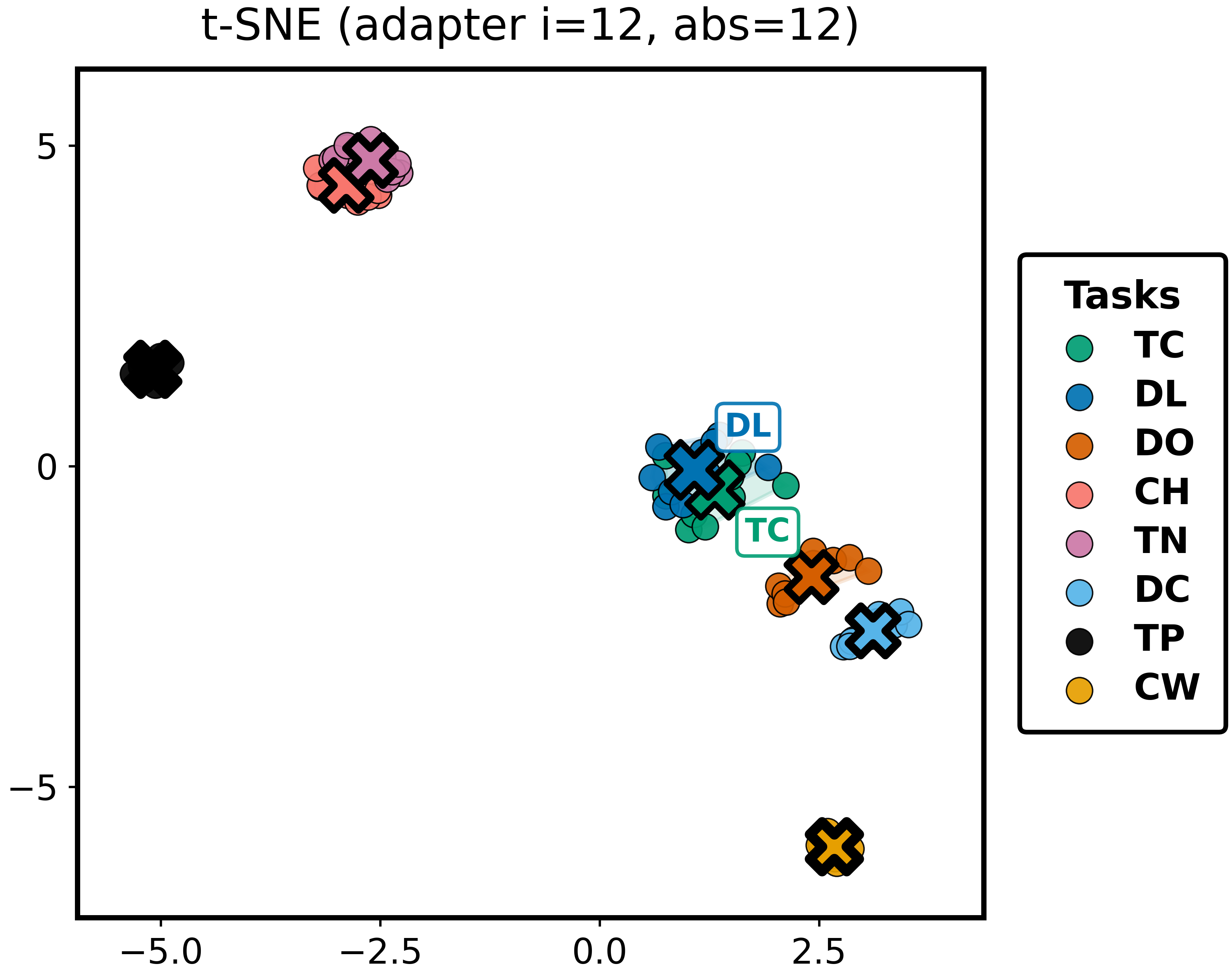}\hfill
  \includegraphics[width=0.26\textwidth]{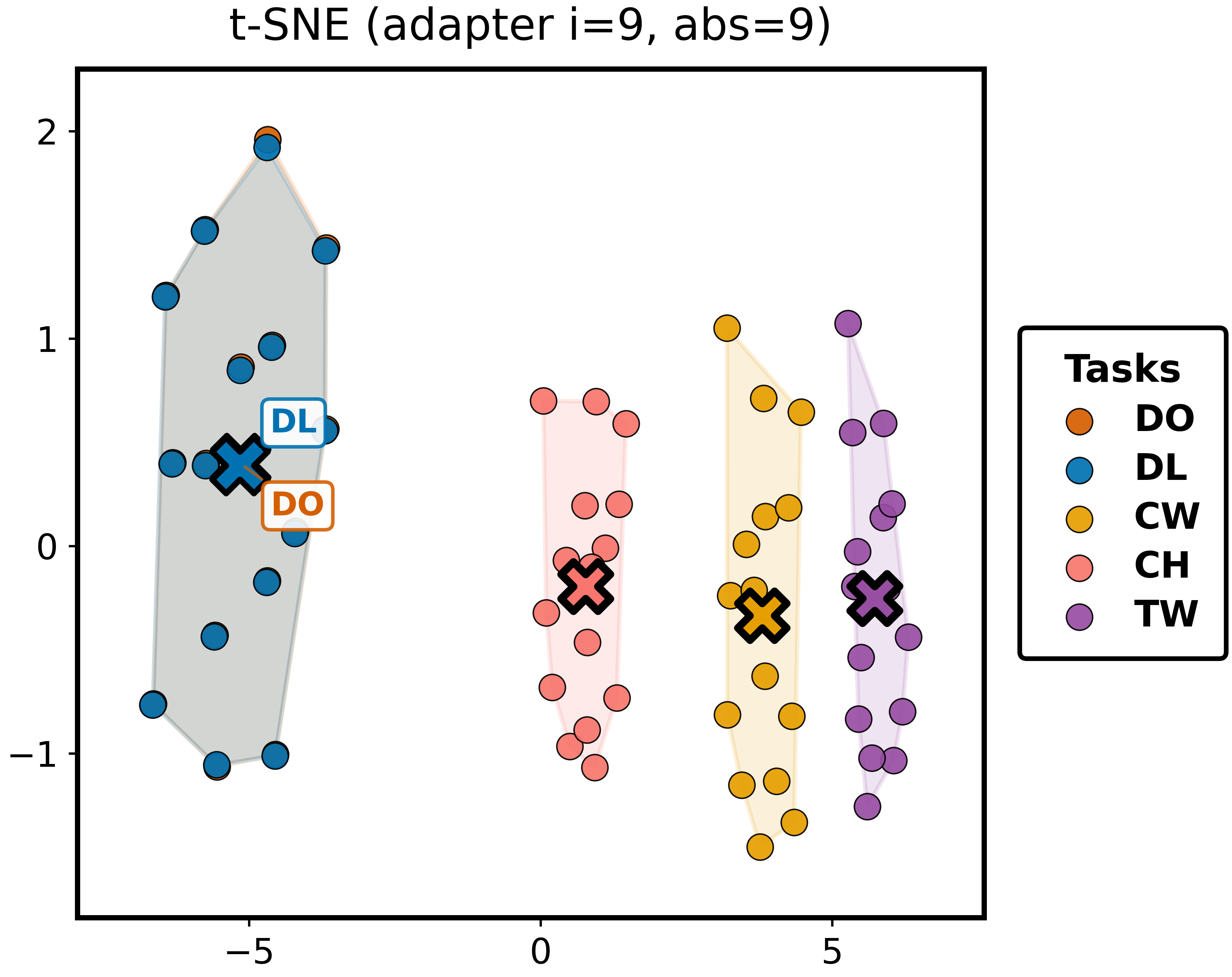}

  \caption{\textbf{Token analysis.}
    \textbf{Left}: Mean image--video token similarity across adapter layers after Visual7W$\rightarrow$NExT-QA continual training.
    \textbf{Middle/Right}: t-SNE visualisations of task-specific prompt tokens at the semantically most-discriminative mid layer ($i=12$ for NExT-QA, $i=9$ for DramaQA).}
  \label{fig:token_corr}
\end{figure*}

\subsection{ImageQA $\rightarrow$ VideoQA}
\label{subsec:main_imageqa_videoqa}

\subsubsection{Setup}

\paragraph{Dataset.}
We use Visual7W~\cite{zhu2016visual7w} (28K images, 140K QA pairs) as the ImageQA source and pair it with NExT-QA to form an ImageQA\,$\rightarrow$\,VideoQA continual sequence: we first pretrain offline on Visual7W, then continue under the original 8-task NExT-QA protocol, probing how image-level pretraining affects downstream video reasoning under continual adaptation.

% Dataset:
% - Visual7W: we split 80/20 for train/val 
% - 28653 images 
% Raw images: 28,653
% Raw QA pairs: 139,868
% After drop_empty images: 28,653
% After drop_empty QA pairs: 139,868
% ------------------------------------------
% Train images: 22,922 | Train QA pairs: 111,913
% Val   images: 5,731 | Val   QA pairs: 27,955

%Climb -> NExT-QA, Climb -> STAR

\paragraph{Baselines \& Metrics.}
LLM-based PEA baselines have not been evaluated in this setting to our knowledge. We compare \model~against the next-best baseline Bisecle, treating each image as a single-frame video by duplicating its visual tokens to match the video token length, which yields a unified continual fine-tuning protocol across image and video tasks.

\paragraph{Implementation.}
We retain the backbone and hyperparameters from the NExT-QA setup. On Visual7W, we fine-tune for 10 epochs with $\mathcal{L}_{\mathrm{NLL}}$ and $\mathcal{L}_{\mathrm{VAQ}}$ only, disabling video-specific auxiliary losses since static images require no temporal reasoning. We use an 80/20 train--validation split.

\subsubsection{Results}

Table~\ref{tab:v7w_to_nextqa_results} shows that \model~outperforms Bisecle on Visual7W\,$\rightarrow$\,NExT-QA in both \ACC{} and \FOG{}, achieving $43.44$ vs $41.90$ on Visual7W and $62.31$ vs $58.24$ on NExT-QA, with consistently lower forgetting. More importantly, continual ImageQA\,$\rightarrow$\,VideoQA is markedly harder than VideoQA alone and tends to induce \emph{slight negative transfer}: Bisecle's NExT-QA accuracy collapses from $62.37$ (VideoQA-only) to $58.24$ ($-4.13$), whereas \model~degrades only mildly ($64.11 \rightarrow 62.31$, $-1.80$), retaining over half of the gap as a clear advantage.

We attribute this robustness to a depth-wise specialisation mismatch revealed in Fig.~\ref{fig:token_corr}~(left). Prompt similarity between ImageQA and VideoQA tasks peaks around layers $5$--$11$ (mean $\approx 0.7$) and then steadily declines toward the final layers (mean $\approx 0.3$), indicating that mid-layer semantics are largely shared while \emph{late layers specialise for temporal reasoning}. Strictly anchoring late-layer behaviour to the static-scene prior learned from ImageQA would therefore freeze the very plasticity that VideoQA most needs. Guided by this observation, we relax LA-Reg on the ImageQA task during the ImageQA\,$\rightarrow$\,VideoQA transition, letting the generator reshape temporal-specific layers while the shared task code preserves reusable mid-layer structure. This selective relaxation is uniquely enabled by \model's layer-conditioned generation and underlies its consistent advantage under cross-modal continual transfer.

% Table~\ref{tab:v7w_to_nextqa_results} shows \model~ clearly outperforms Bisecle on Visual7W$\rightarrow$NExT-QA in both \ACC~and \FOG. Importantly, continual ImageQA$\rightarrow$VideoQA is harder than continual VideoQA alone and can induce negative transfer: Bisecle drops sharply (e.g., $62.37\%\rightarrow55.32\%)$, while \model~ degrades only mildly (4.68\% accuracy).

% We hypothesise this comes from a depth-wise mismatch: ImageQA pretraining strengthens shared mid-layer semantics, but introduces a static-scene bias and reduces late-layer plasticity needed for temporal reasoning. Our token analysis (Fig.~\ref{fig:token_corr} left) supports this: ImageQA and VideoQA prompt similarity peaks in the mid layers and then declines toward the final layers, where temporal specialisation emerges. To avoid freezing temporal learning, we disable forgetting constraints during transfer; even then, \model~ is more robust to the task switch, retaining over $7\%$ higher Visual7W accuracy at the end of continual training.

\begin{table}[!t]
\centering
\caption{\emph{Continual ImageQA $\rightarrow$ VideoQA} results on \VisualWcap $\rightarrow$ \NextQAcap, with accuracy and forgetting.
\bestcap{Bold} and \runnercap{underline} denote the best and second-best results.}
\label{tab:v7w_to_nextqa_results}

\begingroup
% Make \venue smaller ONLY inside this table
\renewcommand{\venue}[1]{\textcolor{black!60}{\textsf{(\fontsize{6}{6.8}\selectfont #1)}}}

\fontsize{7}{8}\selectfont
\setlength{\tabcolsep}{2.5pt}
\renewcommand{\arraystretch}{1.2}

\begin{tabular}{L{3.5cm}*{4}{c}}
\toprule
\textbf{Method ~\venue{Venue}} &
\multicolumn{2}{c}{\cellcolor{V7WHeaderBG}\rule{0pt}{2.2ex}\textbf{Visual7W}} &
\hdr{NextQA}{NExT\mbox{-}QA} \\
\cmidrule(lr){2-3}\cmidrule(lr){4-5}
& \HACC & \HFOG & \HACC & \HFOG \\
\midrule
\rowcolor{black!3}Bisecle ~\venue{NeurIPS'25}
& 41.90 & 28.89 & 58.24 & 5.93 \\
\textbf{\model~ ~\venue{ours}}
& \best{43.44} & \best{27.92} & \best{62.31} & \best{5.02} \\
\bottomrule
\end{tabular}

\endgroup
\end{table}

% Give qualitative examples

% ImageQA -> VideoQA then test on Video, Visual
% VideoQA -> ImageQA then test on Visual, Video

\section{Conclusion}
\label{sec:Conclusion}

We presented \model{}, a transformer-based token hyper-generator for continual VideoQA that synthesises fine-tuning tokens on demand, enabling \emph{explicit, memory-bounded} adaptation with a fixed budget. To mitigate forgetting, we introduced a look-ahead regulariser that anchors the generator to prior tasks, with a sharpness-aware interpretation explaining its improved retention via flatter cross-task minima. A causal view of VideoQA further guided the weighting of auxiliary objectives to strengthen cross-modal alignment without anti-causal drift. Across standard continual VideoQA benchmarks, \model~consistently improves accuracy and reduces forgetting, providing a strong foundation for lifelong multimodal learning.

\section*{Limitations}
\paragraph{No positive cross-modal transfer.}
On ImageQA$\rightarrow$VideoQA, \model{} suffers far less negative transfer than the strongest baseline ($-1.80$ vs $-4.13$ \ACC{}), but image pretraining does not yet yield a net gain on video reasoning, which remains open for any PEA framework.

\paragraph{Training cost.}
The look-ahead inner loop increases training time and peak memory relative to simpler PEA baselines (Table~\ref{tab:compute_cost}). The overhead does not recur at deployment and the $1$-step variant recovers most of it, but training remains more expensive.

\paragraph{Scope of distribution shift.}
The standard NExT-QA and DramaQA splits are question-type-incremental. Our domain- and modality-incremental protocols go further, but streams with broader visual and temporal drift remain to be evaluated.

\paragraph{Routing and horizon.}
Nearest-key retrieval is less accurate on visually homogeneous streams (e.g., $41.7\%$ on DramaQA); its downstream effect is small by design ($0.50$ \ACC{} below oracle routing) but not eliminated. LA-Reg scales linearly with the number of stored codes, which is negligible at our scale and can be held constant by subsampling codes (Appendix~\ref{app:efficiency}); our streams span up to $13$ tasks, and longer horizons are left for future work.

\paragraph{Future work.}
Extensions include audio-visual reasoning, embodied perception, and long-horizon multi-task agents, where compact task-code memory and on-demand prompt synthesis suit heterogeneous task streams.

\section*{Acknowledgements}
This work was supported by the US Army International Technology Center Pacific (ITC-IPAC) under Contract No.\ FA520923C0020. This research is also supported by the ARC Training Centre for Whole Life Design of Carbon Neutral Infrastructure (IC230100015). The authors acknowledge ResetData for computational support, the National Computational Infrastructure (NCI Australia), an NCRIS-enabled capability supported by the Australian Government, for the computational resources used in this study, and the Katana computational cluster, supported by Research Technology Services at UNSW Sydney.

\newpage

\bibliography{hypertokens}

\newpage

\appendix
\section{Extended Preliminaries}
\subsection{LLaMA-Adapter}
\label{sec:extended_background_pea}
Following prior work~\cite{tan2025bisecle, cai2024empowering}, we adopt \emph{LLaMA-Adapter}~\cite{zhang2024llamaadapter} for parameter-efficient continual adaptation of a frozen multimodal LLM on VideoQA. For each task $t$, we introduce a small set of adapter tokens at selected transformer layers,
\begin{equation}
P_l^{t}\in\mathbb{R}^{N_{p}\times C},\quad l=1,\dots,L_{p},
\end{equation}
injected into the frozen backbone and optimised under the objective in Eq.~\ref{eq:nll_sum_ht}. In \model, the token generator $H_\phi$ synthesises layer-wise adapter tokens $\{P_l^t\}$ from compact task codes, so continual fine-tuning reduces to updating $\phi$ rather than maintaining per-task adapter banks.

\paragraph{Gated prefix attention.}
At layer $l$, let $X_l \in \mathbb{R}^{N \times C}$ be the hidden sequence of the frozen backbone (text tokens, preceded by visual tokens from a projector $W_{\mathrm{vis}}$ shared across tasks in the multimodal case). Each token in $X_l$ attends to two sources: the other tokens of $X_l$, as in the original model, and the $N_p$ adapter tokens $P_l^t$. Writing $\mathrm{Attn}(A, B)$ for attention with queries from $A$ and keys/values from $B$, the layer computes
\begin{equation}
\label{eq:gated_prefix}
\begin{split}
X_l' &= X_l + \mathrm{Attn}(X_l, X_l) \\
     &\quad + \tanh(g_l) \odot \mathrm{Attn}(X_l, P_l^t), \\
X_{l+1} &= X_l' + \mathrm{FFN}(X_l').
\end{split}
\end{equation}
The two attention terms are normalised by separate softmaxes, so the first term is exactly the frozen model's own self-attention and is never altered by the prefix. The second term injects the adapter tokens, scaled by a learnable per-head gate $g_l \in \mathbb{R}^{H}$ passed through $\tanh$. Following LLaMA-Adapter, $g_l$ is initialised to zero: training starts from the unmodified backbone and the influence of the prompts grows only as the gates open. Setting $g_l = 0$ for all $l$ therefore removes the generated prompts entirely, which we use as an ablation in Appendix~\ref{app:code_swap}.

\section{Theoretical Results}
\label{sec:appendix_theoretical_results}

\subsection{Proof of Theorem~\ref{thm:la_reg_sam}}

Before presenting the proof, we briefly recall the intuition behind \emph{sharpness-aware minimisation} (SAM)~\cite{foret2021sharpnessaware}.

\paragraph{SAM and flat minima.}
Given a loss $\mathcal{L}(\theta)$, SAM minimises the robust objective
\begin{equation}
\min_{\theta}\;\max_{\|\epsilon\|\le \rho}\;\mathcal{L}(\theta+\epsilon),
\label{eq:sam_obj}
\end{equation}
seeking parameters whose loss stays small under worst-case perturbations within a ball of radius $\rho$. A first-order analysis gives the inner maximiser $\epsilon^\star = \rho\,\nabla_\theta \mathcal{L}(\theta) / \|\nabla_\theta \mathcal{L}(\theta)\|$, so SAM penalises high-curvature directions and favours flat minima, which generalise better under distribution shift. This robustness perspective motivates LA-Reg, which can be read as enforcing prediction consistency on past task codes under a look-ahead parameter perturbation.

\paragraph{Proof setup.}
Fix $\tau<t$ and write
\begin{equation}
H_\phi^\tau \coloneqq H_\phi(z^\tau),\quad J_\phi^\tau \coloneqq J_\phi(z^\tau),
\end{equation}
\begin{equation}
a_\tau \coloneqq H_\phi^\tau - H_{\phi^\star}^\tau,\quad b_\tau \coloneqq J_\phi^\tau\,\Delta\phi.
\end{equation}

\begin{proof}
By assumption~(A1), $H_\phi^\tau$ admits the second-order expansion
\begin{equation}
H_{\phi+\Delta\phi}^\tau = H_\phi^\tau + J_\phi^\tau \Delta\phi + r_\tau,
\end{equation}
with $\|r_\tau\|_2 \le C_H \|\Delta\phi\|_2^2$. Hence
\begin{equation}
H_{\phi+\Delta\phi}^\tau - H_{\phi^\star}^\tau = a_\tau + b_\tau + r_\tau,
\end{equation}
and substituting into Eq.~\eqref{eq:la_reg_h} gives
\begin{equation}
\label{eq:LA_split_sum}
\mathcal{L}_{\mathrm{LA\text{-}Reg}}^{t}(\phi) = \sum_{\tau=1}^{t-1} \|a_\tau + b_\tau + r_\tau\|_2^2.
\end{equation}

\paragraph{Leading term.}
Since $\Delta\phi = -\eta\,g(\phi)$, we have $b_\tau = -\eta\,J_\phi^\tau g(\phi)$, so
\begin{equation}
\label{eq:leading_term}
\sum_{\tau=1}^{t-1} \|b_\tau\|_2^2 = \eta^2 \sum_{\tau=1}^{t-1} \|J_\phi(z^\tau)\,g(\phi)\|_2^2.
\end{equation}

\paragraph{Residual bound.}
Expanding the squared norm of $a_\tau + b_\tau + r_\tau$ and subtracting $\|b_\tau\|_2^2$, Cauchy--Schwarz yields
\begin{equation}
\label{eq:diff_bound_triangle}
\begin{split}
&\Bigl| \mathcal{L}_{\mathrm{LA\text{-}Reg}}^{t}(\phi) - \sum_{\tau<t} \|b_\tau\|_2^2 \Bigr| \\
&\;\le\; \sum_{\tau<t} \Bigl( \|a_\tau\|_2^2 + \|r_\tau\|_2^2 \\
&\qquad\qquad + 2\|a_\tau\|_2\|b_\tau\|_2 \\
&\qquad\qquad + 2\|a_\tau\|_2\|r_\tau\|_2 \\
&\qquad\qquad + 2\|b_\tau\|_2\|r_\tau\|_2 \Bigr).
\end{split}
\end{equation}
Assumptions (A1)--(A3) give the per-term bounds
\begin{equation}
\begin{aligned}
\|a_\tau\|_2 &\le \varepsilon,\\
\|b_\tau\|_2 &\le L_J\,\eta G,\\
\|r_\tau\|_2 &\le C_H \eta^2 G^2.
\end{aligned}
\end{equation}
Plugging these into Eq.~\eqref{eq:diff_bound_triangle},
\begin{equation}
\label{eq:per_tau_bound}
\begin{split}
&\Bigl| \mathcal{L}_{\mathrm{LA\text{-}Reg}}^{t}(\phi) - \sum_{\tau<t} \|b_\tau\|_2^2 \Bigr| \\
&\;\le\; (t-1) \Bigl( \varepsilon^2 + 2 L_J G\,\varepsilon\eta + 2 C_H G^2\,\varepsilon\eta^2 \\
&\qquad\qquad + 2 C_H L_J G^3\,\eta^3 + C_H^2 G^4\,\eta^4 \Bigr).
\end{split}
\end{equation}

\paragraph{Final simplification.}
For sufficiently small $\eta$ (say $\eta \le 1$), we have $\eta^2 \le \eta$ and $\eta^4 \le \eta^3$. Hence Eq.~\eqref{eq:per_tau_bound} simplifies to
\begin{equation}
\label{eq:final_bound_simplified_1}
\Bigl| \mathcal{L}_{\mathrm{LA\text{-}Reg}}^{t}(\phi) - \sum_{\tau<t} \|b_\tau\|_2^2 \Bigr| \le C(\varepsilon^2 + \varepsilon\eta + \eta^3),
\end{equation}
with
\begin{equation}
\begin{split}
C \coloneqq (t-1) \max\bigl\{ &1, \\
&\, 2 L_J G + 2 C_H G^2, \\
&\, 2 C_H L_J G^3 + C_H^2 G^4 \bigr\}.
\end{split}
\end{equation}
which is independent of $\eta$. Substituting Eq.~\eqref{eq:leading_term} into Eq.~\eqref{eq:final_bound_simplified_1} yields Eq.~\eqref{eq:la_sam_expansion}.
\end{proof}

\subsection{Adaptation on a Shared Prompt Manifold}
\label{sec:appendix_manifold}

We formalise the claim in Sec.~\ref{sec:theory} that the token generator $H_\phi$ confines local task-specific adaptation to a low-dimensional subspace of the full prompt space.

\paragraph{Setup.}
For clarity, we view the output of the hypernetwork as a vector in $\mathbb{R}^{d_P}$, where $d_P \coloneqq L_p N_p C$, and similarly flatten task codes to $\mathbb{R}^{d_z}$ with $d_z \coloneqq N_z C_z$. The hypernetwork is then a map
\begin{equation}
H_\phi: \mathbb{R}^{d_z} \rightarrow \mathbb{R}^{d_P}.
\end{equation}
Let $J_z H_\phi(z) \in \mathbb{R}^{d_P \times d_z}$ denote its Jacobian with respect to the task code, and define the reachable prompt set
\begin{equation}
\mathcal{M}_\phi \coloneqq \{ H_\phi(z) : z \in \mathbb{R}^{d_z} \} \subseteq \mathbb{R}^{d_P}.
\end{equation}

\begin{proposition}[Low-dimensional adaptation subspace]
\label{prop:manifold}
Assume $H_\phi$ is continuously differentiable. For any task code $z^t$ and perturbation $\delta z^t$, the induced prompt change $\delta P^t \coloneqq H_\phi(z^t + \delta z^t) - H_\phi(z^t)$ satisfies
\begin{equation}
\label{eq:prop_manifold}
\delta P^t = J_z H_\phi(z^t)\,\delta z^t + R(\delta z^t),
\end{equation}
with $\|R(\delta z^t)\|_2 = O(\|\delta z^t\|_2^2)$. The first-order term lies in $\mathrm{Range}(J_z H_\phi(z^t))$, a subspace of dimension at most $d_z = N_z C_z$, strictly smaller than $d_P = L_p N_p C$ whenever $d_z < d_P$.
\end{proposition}

\begin{proof}
By continuous differentiability of $H_\phi$ in $z$, Taylor's theorem at $z^t$ yields Eq.~\eqref{eq:prop_manifold} with the stated remainder bound. The first-order term $J_z H_\phi(z^t)\,\delta z^t$ lies in $\mathrm{Range}(J_z H_\phi(z^t))$ by definition, and the dimension of this range equals the rank of the Jacobian, which is bounded by the dimension of its domain:
\begin{equation}
\begin{split}
\dim \mathrm{Range}\bigl(&J_z H_\phi(z^t)\bigr) \\
&= \mathrm{rank}\bigl(J_z H_\phi(z^t)\bigr) \\
&\le \min(d_P, d_z) = d_z,
\end{split}
\end{equation}
where the last equality uses $d_z \le d_P$ in our configuration.
\end{proof}

\paragraph{Interpretation.}
Proposition~\ref{prop:manifold} formalises why $H_\phi$ acts as an implicit structural regulariser. Although the ambient prompt space has $d_P$ degrees of freedom, task-specific updates through the compact code $z^t$ can only move, to first order, within a subspace of dimension at most $d_z \ll d_P$. \model~therefore biases task prompts toward a shared low-dimensional manifold while still permitting fine-grained task-specific control through the task code, discouraging task-specific overfitting and reducing the cross-task interference inherent in unconstrained prompt tuning. In our default LLaMA-2-7B configuration, $d_z \approx 7.7 \!\times\! 10^3$ and $d_P \approx 1.3 \!\times\! 10^6$, corresponding to roughly a $170\times$ compression of the local adaptation space (the per-task \emph{storage} is smaller still, ${\approx}200\times$ below the prompt tensor, since $z_{\mathrm{inv}}$ is shared across tasks; Sec.~\ref{sec:Method}).

\subsection{Empirical Verification in a Controlled Setting}
\label{app:toy}
Theorem~\ref{thm:la_reg_sam} predicts that LA-Reg steers $H_\phi$ toward regions where past-task outputs are insensitive to gradient-aligned perturbations, i.e., flatter cross-task solutions. To test this prediction directly, isolated from the LLM and the VideoQA objective, we construct a controlled five-task continual-learning problem on the generator alone.

\paragraph{Setup.}
Each task $\tau$ is a fixed pair $(z^\tau, y^\tau)$ with $z^\tau \in \mathbb{R}^{16}$ and $y^\tau \in \mathbb{R}^{32}$ drawn from a standard Gaussian, learned sequentially with an MSE objective. $H_\phi$ is a three-layer MLP ($16 \!\to\! 64 \!\to\! 64 \!\to\! 32$, $\tanh$ activations, ${\sim}7$K parameters). Using matched initialisation we compare (i) the task loss alone and (ii) the task loss plus LA-Reg, $\sum_{\tau<t}\|H_\phi(z^\tau) - H_{\phi^\star}(z^\tau)\|_2^2$, evaluated at a two-step look-ahead point exactly as in Eq.~\ref{eq:la_reg_h}. We report two quantities. $\Delta\mathcal{L}@\rho$ is the increase in total-task loss after perturbing $\phi$ by $\rho\!=\!0.05$ along the gradient direction; lower values indicate a flatter solution. \emph{Forgetting} is the mean loss on all earlier tasks at the end of the sequence; since each task is trained to convergence before the next begins, this equals the loss increase caused by subsequent training.

\begin{table}[h]
\centering\small
\begin{tabular}{lcc}
\toprule
Method & $\Delta\mathcal{L}@\rho$ ($\downarrow$) & Forgetting ($\downarrow$) \\
\midrule
w/o LA-Reg & $1.44 \pm 0.31$ & $26.0 \pm 7.2$ \\
w/ LA-Reg & $0.02 \pm 0.003$ & $\approx 0$ \\
\bottomrule
\end{tabular}
\caption{\textbf{Controlled verification of Theorem~\ref{thm:la_reg_sam}} (mean $\pm$ std over 10 seeds). LA-Reg reduces the sharpness measure by over $60\times$ and nearly eliminates forgetting.}
\label{tab:toy_flatness}
\end{table}

\paragraph{Result.}
Table~\ref{tab:toy_flatness} shows that LA-Reg reduces the sharpness measure by over $60\times$ and drives forgetting to nearly zero, with both quantities decreasing consistently across all 10 seeds. This directly supports the link between flatter cross-task solutions and improved retention, and mirrors the trend on NExT-QA, where forgetting falls from $8.83$ to $4.76$ and decreases further with look-ahead depth (Table~\ref{tab:ablation_look_ahead}). The script for this experiment is included in our code release.

\section{Additional Experimental Details}
\label{sec:additional_exp_details}

\subsection{Architecture}
\label{app:details}

\paragraph{$H_\phi$ implementation.}
We use the same generator across all three datasets (NExT-QA, DramaQA, Visual7W). Its computation proceeds in four steps.
\emph{(i) Code projection.} The task code $z^t \in \mathbb{R}^{N_z \times C_z}$, with $N_z = 30$ slots ($5$ invariant, $25$ task-specific) of width $C_z = 256$, is projected to the generator's internal width $d = 384$ and given a positional embedding, yielding $X^t$.
\emph{(ii) Layer conditioning.} A table of $L_p = 32$ learnable layer embeddings $\{e_\ell\}$, one per adapter layer, predicts a FiLM scale and shift, $X \odot (1{+}\tanh\gamma_\ell) + \beta_\ell$, which is applied to $X^t$ and to the query tokens below; this is how a single generator produces different prompts for different layers.
\emph{(iii) Query initialisation.} The $N_p$ initial query tokens are derived from the code: an attention-pooling head summarises $X^t$ into $N_p$ vectors, each passes through a rank-$r{=}32$ bottleneck onto a shared basis, and a learned position bias is added; a first cross-attention from these queries to $X^t$ gives $U^{(0)}$.
\emph{(iv) Refinement and output.} $U^{(0)}$ is refined by $2$ pre-norm blocks of cross-attention (to $X^t$), self-attention, and MLP, each with $8$ heads and dropout $0.2$, and a shared linear head $W_o$ maps the result from $d = 384$ to the LLaMA-2-7B hidden dimension $C = 4096$, producing the $N_p$ prompt tokens for layer $\ell$.

\subsection{Hyper-parameters and Continual Fine-tuning Details}
\label{sec:hyperparams}

\paragraph{Routing.}
The task-key dimension is set to $D_k = 256$. For efficiency, we maintain an exponential moving average of the task key updated batch-wise with momentum $0.95$, which closely approximates the dataset-level mean (Eq.~\ref{eq:task_key}) at a fraction of the cost. The EMA introduces neither learnable parameters nor gradient updates; computing the exact key requires an additional full pass over the training set, whereas the same batch embeddings are already available during training. Table~\ref{tab:ema_vs_exact} confirms the approximation is essentially lossless, as expected since the embeddings come from frozen modules and do not evolve during training.

\begin{table}[h]
\centering\small
\begin{tabular}{lccc}
\toprule
Key computation & Routing Acc & \ACC{} & \FOG{} \\
\midrule
Exact mean (Eq.~\ref{eq:task_key}) & 69.76\% & 64.14 & 4.69 \\
EMA (momentum $0.95$) & 69.00\% & 64.11 & 4.76 \\
\bottomrule
\end{tabular}
\caption{\textbf{EMA vs.\ exact task keys} on NExT-QA. The online EMA matches the exact dataset-level mean within $0.03$ \ACC{}.}
\label{tab:ema_vs_exact}
\end{table}

\paragraph{Training schedule.}
We fine-tune on NVIDIA H200 GPUs with a two-epoch warm-up. Epoch budgets are matched across methods: five epochs for all methods on NExT-QA, and six epochs for all methods in the continual-dataset setting (NExT-QA$\rightarrow$DramaQA), which inherits the DramaQA configuration end-to-end. The only exception is the single-dataset DramaQA setting, where \model{} uses six epochs while the baselines are reported at their published five-epoch results~\cite{tan2025bisecle, cai2024empowering}; to verify this does not tilt the comparison, we re-ran the baselines for six epochs and observed no material change in accuracy, as they had already converged by epoch five. The NExT-QA stage of the Visual7W$\rightarrow$NExT-QA transfer reuses the standard NExT-QA configuration.

\paragraph{Pretraining on Visual7W.}
We pretrain both Bisecle and \model~for ten epochs with a four-epoch warm-up, using AdamW with batch size $16$ and gradient accumulation over $64$ steps; all remaining hyperparameters follow the protocol above. We select the checkpoint with the best validation accuracy as the starting point for continual adaptation on NExT-QA.

\paragraph{Summary of hyperparameters.}
Table~\ref{tab:hparams} lists every setting used. The \model{} architecture, loss weights, and look-ahead configuration are identical across datasets; the backbone-side training recipe (learning rate, weight decay, adapter length, sequence length, and epochs) follows the published per-dataset recipe of Bisecle~\cite{tan2025bisecle} so that all methods share the same schedule. The \emph{base learning rate} applies to the backbone-side trainable modules of LLaMA-Adapter, namely the visual projector (which maps CLIP features into the LLaMA hidden space) and the adapter gates; the \model{} modules use their own smaller rates.

\begin{table}[h]
\centering\small
\setlength{\tabcolsep}{4pt}
\resizebox{\columnwidth}{!}{%
\begin{tabular}{@{}lcc@{}}
\toprule
& \textbf{NExT-QA} & \textbf{DramaQA} \\
\midrule
\multicolumn{3}{@{}l}{\emph{Shared across datasets}} \\
Backbone / visual encoder & \multicolumn{2}{c}{LLaMA-2-7B / CLIP ViT-L/14} \\
Adapter layers $L_p$ & \multicolumn{2}{c}{32} \\
Task code $N_z$ ($N_{\mathrm{inv}}{+}N_{\mathrm{sp}}$) / $C_z$ & \multicolumn{2}{c}{30 (5{+}25) / 256} \\
Generator $H_\phi$ & \multicolumn{2}{c}{2 blocks, width 384, 8 heads, dropout 0.2} \\
Key dim.\ $D_k$ / EMA momentum & \multicolumn{2}{c}{256 / 0.95} \\
$\lambda_{\mathrm{LA\text{-}Reg}}$ / $\lambda_{\mathrm{inv}}$ / $\lambda_{\mathrm{VAQ}}$ / $\lambda_{\mathrm{QAV}}$ & \multicolumn{2}{c}{0.5 / 0.1 / 1.0 / 0.1} \\
Look-ahead steps / inner step size & \multicolumn{2}{c}{2 / $0.5\times$ base LR} \\
Optimiser / $(\beta_1,\beta_2)$ & \multicolumn{2}{c}{AdamW / (0.9, 0.95)} \\
LR: generator $H_\phi$ / $z_{\mathrm{sp}}$ & \multicolumn{2}{c}{$10^{-3}$ / $10^{-3}$} \\
Effective batch size (4 per GPU) & \multicolumn{2}{c}{128} \\
\midrule
\multicolumn{3}{@{}l}{\emph{Per-dataset (following Bisecle)}} \\
Epochs & 5 & 6 \\
Adapter tokens per layer $N_p$ & 10 & 15 \\
Max text length (tokens) & 128 & 280 \\
Base LR (projector, gates) & $5\!\times\!10^{-3}$ & $7.5\!\times\!10^{-3}$ \\
LR: $z_{\mathrm{inv}}$ & $10^{-3}$ & $5\!\times\!10^{-4}$ \\
Weight decay & 0.14 & 0.10 \\
\bottomrule
\end{tabular}}
\caption{\textbf{Complete hyperparameter settings} for \model{}. The continual-dataset stream inherits the DramaQA column; the Visual7W$\rightarrow$NExT-QA stage reuses the NExT-QA column.}
\label{tab:hparams}
\end{table}

\section{Extended Related Work}
\label{sec:more_related_work}
\paragraph{Prompt-based continual learning.}
L2P~\cite{wang2022learning} pioneers prompt retrieval from a fixed pool 
via prompt tuning~\cite{lester2021power} but suffers from interference 
due to limited prompt capacity. DualPrompt~\cite{wang2022dualprompt} 
separates task-invariant (G-Prompt) from task-specific (E-Prompt) 
components, and CODA-Prompt~\cite{smith2023coda} refines this with 
attention-based prompting. LAE~\cite{gao2023unified} couples an online 
PEA module with an offline one via an adaptation-speed calibrator, 
while HiDe-Prompt~\cite{wang2023hierarchical} adopts hierarchical 
prompts with inference-time ensembling. A different line stores per-task 
prompts to avoid interference: ProgPrompt~\cite{razdaibiedina2023progressive} 
concatenates frozen task prompts in the spirit of progressive 
networks~\cite{rusu2016progressive}, and ModalPrompt~\cite{zeng2025modalprompt} 
extends this to multimodal continual instruction tuning by using 
CLIP-based dual-modality similarity for prompt fusion and top-$k$ 
selection. While effective, these methods maintain prompt banks that 
grow linearly with the task stream, and ModalPrompt additionally depends 
on auxiliary CLIP encoders for routing.

\paragraph{Hypernetworks for continual adaptation.}
Hypernetworks~\cite{ha2017hypernetworks} produce the parameters of 
another network conditioned on an auxiliary signal. In continual 
learning, prior work typically conditions on task identity to 
synthesise full or partial model 
weights~\cite{von2020continual,hemati2023partial,chandra2023continual,kumaravelu2025evocl}, 
which is computationally costly and poorly suited to LLM-scale models. 
AdaPrefix++~\cite{adhikari2025adaprefix++} mitigates this cost by 
generating only prefix parameters from trainable per-task embeddings through a shared hypernetwork.
\label{app:adaprefix}
It is conceptually the closest method to \model{}, and we detail the differences here.
\emph{Setting.} AdaPrefix++ targets image classification with a fixed vision backbone, whereas \model{} addresses generative VideoQA with temporal inputs and is evaluated under cross-dataset and cross-modality shifts.
\emph{Routing.} Both can operate without task identity but by different mechanisms: AdaPrefix++ infers the task by running a forward pass through every task's head and selecting the lowest-entropy prediction, an $O(T)$ cost that grows with the number of tasks and depends on model confidence; \model{} routes by a single cosine lookup in a frozen embedding space ($O(1)$), which does not drift as tasks accumulate.
\emph{Architecture.} Both store a compact per-task input to a shared generator (task embeddings vs.\ task codes), but AdaPrefix++ additionally attaches a task-specific adapter to the backbone for every task, whereas \model{} adds no per-task backbone modules; its only per-task state is a routing key and code.
\emph{Regularisation.} AdaPrefix++ regularises against stored past behaviour without look-ahead and provides no theoretical account of its regulariser; \model{} adds LA-Reg, characterised in Theorem~\ref{thm:la_reg_sam} as a cross-task sharpness penalty. Table~\ref{tab:ablation_look_ahead} isolates the contribution of look-ahead alone ($+1.81$ \ACC{}, $-0.64$ \FOG{} from $0$ to $2$ steps).
AdaPrefix++ has no public implementation, so we compare along these design dimensions rather than empirically. Recent hypernetwork-based LLM adaptation 
methods~\cite{charakorn2025texttolora,tan2025instant} target different 
settings (efficient fine-tuning, in-context adaptation) and lack 
continual-learning properties.

\paragraph{Image Question Answering (ImageQA).}
ImageQA is a core multimodal task requiring models to ground visual content in natural-language answers, and recent multimodal LLMs have substantially improved its performance~\cite{Anil2023Gemini, Liu2024Visual}. Yet these systems are predominantly trained offline on stationary data distributions and are not designed for continual learning from evolving multimodal streams. Whether and how they transfer to the more demanding continual VideoQA setting, which adds temporal reasoning on top of distribution shift, remains largely unexplored. We provide one of the first systematic evaluations of this transfer in our ImageQA$\rightarrow$VideoQA protocol (Sec.~\ref{sec:Experiments}).

\paragraph{Meta-Learning for Continual Learning.}
Meta-learning mitigates forgetting by optimising parameters, or update rules, for rapid adaptation with minimal interference. Existing meta-CL methods focus primarily on shaping the optimisation trajectory: MER aligns gradients across replayed experiences via Reptile-style updates~\cite{nichol2018first, riemer2019learning}, and La-MAML extends MAML with a look-ahead objective and per-parameter learning-rate modulation for online continual learning~\cite{gupta2020look}. We target a substantially harder regime, prompt-based continual adaptation of multimodal LLMs for video reasoning, and depart from prior meta-CL work in two ways. First, our look-ahead term regularises the \emph{prompt-output space} rather than the parameter or gradient space, making it directly compatible with frozen-backbone PEA. Second, Theorem~\ref{thm:la_reg_sam} gives a first-order characterisation of our look-ahead term as a cross-task sharpness penalty, connecting it to sharpness-aware minimisation.

% \subsection{Continual VideoQA}

% % \subsubsection{Result with variation}

% % - Different seeds

% \subsection{ImageQA $\rightarrow$ VideoQA}

\section{More Results and Analyses}
\label{sec:more_ablation_studies}

\subsection{Detailed Token Analysis}
\label{app:token_analysis}

\paragraph{Layer selection.}
We select visualisation layers by silhouette score across all 32 adapter layers, picking the middle-layer maximum: $i = 12$ for NExT-QA, $i = 9$ for DramaQA. Early layers process inputs; late layers commit to output generation; middle layers tend to encode the most task-discriminative representations.

\paragraph{NExT-QA (Fig.~\ref{fig:token_corr}, middle).}
Three regions emerge: \textbf{(i)} a \emph{sequential-reasoning} cluster merging TN (Temporal-Next: predicting what happens next) and CH (Causal-How: explaining how an event unfolds), both involving temporal event prediction; \textbf{(ii)} a \emph{descriptive continuum} where DL (Location), TC (Count), DO (Object), and DC (Description) arrange from spatial-temporal grounding to cardinality counting; and \textbf{(iii)} two outliers, TP (Temporal-Previous) and CW (Causal-Why), closer to static-state and abstract-causal reasoning respectively. Interestingly, this organisation does not align with the official NExT-QA taxonomy (Causal/Temporal/Descriptive), suggesting \model{} groups tasks more by reasoning \emph{skill} than by question type.

\paragraph{DramaQA (Fig.~\ref{fig:token_corr}, right).}
A similar skill-oriented pattern appears. DO (Descriptive-Object) and DL (Descriptive-Location) form a spatial-visual grouping (locating an entity arguably requires identifying it first), while CH (Causal-How), CW (Causal-Why), and TW (Temporal-When) occupy distinct strips associated with procedural, abstract-causal, and temporal reasoning. CH and CW are also relatively well separated despite sharing the ``Causal'' taxonomy label, hinting that \model{} captures some distinction between \emph{how} and \emph{why} at the representational level.

\paragraph{Cross-dataset consistency.}
This skill-oriented organisation appears across both datasets, which are trained independently with different question types and orderings. While the visualisations are exploratory, they are consistent with the hypothesis that \model{} captures transferable reasoning structure rather than purely dataset-specific patterns.

\subsection{Impact of $\lambda_{\mathrm{LA\text{-}Reg}}$}
\label{sec:ablate_la_reg_weight}

We sweep $\lambda_{\mathrm{LA\text{-}Reg}}$ over $\{0, 0.25, 0.5, 0.75, 1.0, 1.5\}$ on NExT-QA. Table~\ref{tab:ablate_la_reg_weight} shows that LA-Reg substantially improves over the no-regularisation baseline, peaks at $0.5$--$0.75$, and saturates beyond $1.0$, confirming robustness to a wide range of weights.

\begin{table}[t]
\centering
\caption{\textbf{Effect of $\lambda_{\mathrm{LA\text{-}Reg}}$} on NExT-QA. We use $\lambda_{\mathrm{LA\text{-}Reg}} = 0.5$ as default.}
\label{tab:ablate_la_reg_weight}
\small
\setlength{\tabcolsep}{8pt}
\renewcommand{\arraystretch}{1.1}
\begin{tabular}{c|cc}
\toprule
$\lambda_{\mathrm{LA\text{-}Reg}}$ & \ACC{}~$\uparrow$ & \FOG{}~$\downarrow$ \\
\midrule
$0.00$ & 54.91 & 8.39 \\
$0.25$ & 63.46 &  5.01 \\
\textbf{0.50} & \best{64.11} & \best{4.76} \\
$0.75$ & \runner{63.92} & \runner{4.89} \\
$1.00$ & 63.70 & 4.98 \\
$1.50$ & 63.59 & 5.07 \\
\bottomrule
\end{tabular}
\end{table}

\subsection{Impact of Look-ahead learning rate}

Table~\ref{tab:lookahead_lr} sweeps the look-ahead step size $\eta_{\text{LA}}$ on NExT-QA. \model{} achieves the best \ACC{} ($64.11$) and lowest \FOG{} ($4.76$) at $\eta_{\text{LA}} = 0.50$, and stays within $0.72$ \ACC{} for $\eta_{\text{LA}} \leq 0.75$, indicating robustness to fine-grained tuning. Beyond $\eta_{\text{LA}} \geq 1.0$, both metrics degrade monotonically (up to $-2.30$ \ACC{} and $+1.27$ \FOG{} at $\eta_{\text{LA}} = 1.5$): large inner steps push the probe into the current task's local minimum, so the regulariser encodes \emph{overfitted directions} rather than local flatness, matching our SAM-based motivation. We use $\eta_{\text{LA}} = 0.50$ as default.

\begin{table}[t]
\centering
\caption{\textbf{Effect of look-ahead step size $\eta_{\text{LA}}$ 
on NExT-QA.} Performance is robust for $\eta_{\text{LA}} < 1.0$ and 
degrades for larger values, which over-fit the look-ahead probe to 
the current task. Best in \textbf{bold}, second-best \underline{underlined}.}
\label{tab:lookahead_lr}
\small
\setlength{\tabcolsep}{6pt}
\begin{tabular}{c|cc}
\toprule
$\eta_{\text{LA}}$  & \ACC{}~$\uparrow$ & \FOG{}~$\downarrow$ \\
\midrule
0.10  & 63.74 & 4.92 \\
0.25  & \runner{63.96} & \runner{4.84}  \\
\textbf{0.50}  & \best{64.11} & \best{4.76} \\
0.75  & 63.39 & 5.07 \\
1.00  & 62.72 & 5.23 \\
1.50  & 61.81 & 6.03 \\
\bottomrule
\end{tabular}
\end{table}

\subsection{Impact of Video Auxiliary Loss}
\label{sec:abl_qav}

Our causal analysis (Sec.~\ref{sec:Method}) treats $\mathcal{L}_{\mathrm{QAV}}$ as an anti-causal but useful grounding signal, prescribing a small $\lambda_{\mathrm{QAV}}$. Table~\ref{tab:qav_weight} confirms the predicted \emph{inverted-U}: removing QAV ($\lambda_{\mathrm{QAV}} = 0$) drops \ACC{} to $63.75$ as visual grounding is lost, the optimum at $\lambda_{\mathrm{QAV}} = 0.10$ reaches $64.11$ \ACC{} and $4.76$ \FOG{}, and up-weighting progressively degrades to $62.21$ \ACC{} at $\lambda_{\mathrm{QAV}} = 1.50$ ($-1.90$ \ACC{}, $+0.73$ \FOG{}) as the underdetermined $p(V \mid Q, A)$ pulls the model along anti-causal paths. Even the $\lambda_{\mathrm{QAV}} = 0$ result highlights the importance of the causal $\mathcal{L}_{\mathrm{VAQ}}$, which alone lifts \ACC{} from $62.41$ (Table~\ref{tab:ablate_losses_ticks}, no auxiliary) to $63.75$. The same trend is reported in concurrent VideoQA work~\cite{tan2025bisecle}. We adopt $\lambda_{\mathrm{QAV}} = 0.10$ throughout.

\begin{table}[t]
\centering
\caption{\textbf{Effect of QAV auxiliary loss weight $\lambda_{\mathrm{QAV}}$ 
on NExT-QA.} Small weights ($\lambda_{\mathrm{QAV}} \leq 0.25$) yield best 
performance; disabling QAV ($\lambda_{\mathrm{QAV}} = 0$) loses cross-modal 
supervision, while large weights induce anti-causal drift as predicted by 
our causal analysis. Best in \textbf{bold}.}
\label{tab:qav_weight}
\small
\setlength{\tabcolsep}{6pt}
\begin{tabular}{c|cc}
\toprule
$\lambda_{\mathrm{QAV}}$  & \ACC{}~$\uparrow$ & \FOG{}~$\downarrow$ \\
\midrule
0.00  & 63.75 & 4.98 \\
\midrule
\textbf{0.10}  & \best{64.11} & \best{4.76} \\
0.25  & \runner{64.02} & \runner{4.77} \\
0.50  & 63.42 & 5.10 \\
0.75  & 63.58 & 5.03 \\
1.00  & 62.68 & 5.24 \\
1.50  & 62.21 & 5.49 \\
\bottomrule
\end{tabular}
\end{table}

\begin{figure*}[t]
\centering
\includegraphics[width=\textwidth]{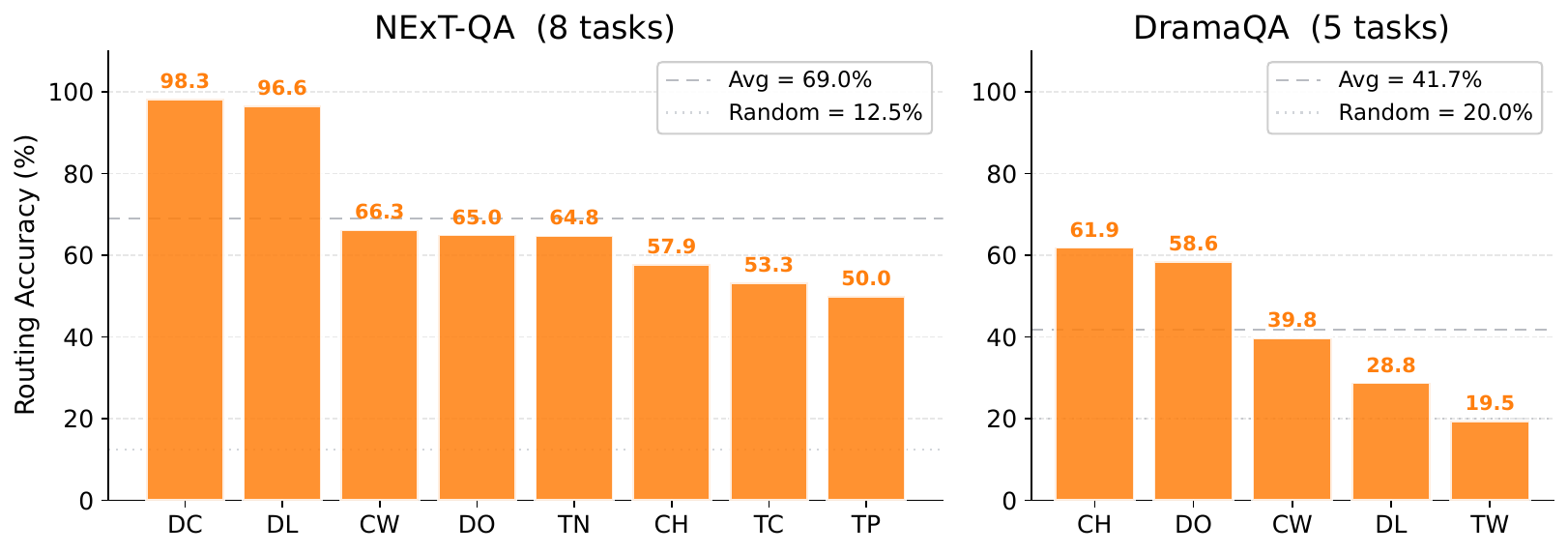}
\caption{\textbf{Per-task accuracy of our training-free routing} on NExT-QA (left, $8$ tasks) and DramaQA (right, $5$ tasks). Bars show retrieval accuracy via cosine similarity in the frozen embedding space; dashed and dotted lines mark the dataset-wise average and the random baseline.}
\label{fig:routing_acc}
\end{figure*}

\begin{table}[t]
\centering
\caption{\textbf{Training-free vs oracle routing} for \model{}. Minor gaps on both benchmarks confirm that downstream performance is largely decoupled from routing accuracy.}
\label{tab:oracle_routing}
\small
\setlength{\tabcolsep}{6pt}
\renewcommand{\arraystretch}{1.1}
\begin{tabular}{l|cc|cc}
\toprule
& \multicolumn{2}{c|}{\textbf{NExT-QA}} & \multicolumn{2}{c}{\textbf{DramaQA}} \\
Routing & \ACC{}~$\uparrow$ & \FOG{}~$\downarrow$ & \ACC{}~$\uparrow$ & \FOG{}~$\downarrow$ \\
\midrule
Training-free (ours) & $64.11$ & $4.76$ & $72.52$ & $9.76$ \\
Oracle               & $64.61$ & $4.14$ & $72.90$ & $9.01$ \\
\midrule
$\Delta$             & $-0.50$ & $+0.62$ & $-0.38$ & $+0.75$ \\
\bottomrule
\end{tabular}
\end{table}

\subsection{Routing Accuracy.}
\label{app:routing_acc}

Fig.~\ref{fig:routing_acc} reports per-task accuracy of our gradient-free routing. NExT-QA reaches $69.0\%$ on average ($5.5\times$ random); DramaQA drops to $41.7\%$ ($2\times$ random) due to heavy overlap among its five Wh-categories sharing a single TV-series visual stream. Nearest-key errors are thus systematic rather than random: they concentrate among tasks that share visual context, so a misrouted query lands on a semantically adjacent task whose code produces a similar prompt. Crucially, this gap does not translate to downstream loss: against an \emph{oracle} variant with ground-truth task identity (Table~\ref{tab:oracle_routing}), \model~loses only $0.50$ \ACC{} on NExT-QA and $0.38$ on DramaQA. The reason is that discriminability lives in \emph{token space}, not in pretrained key space: although DramaQA's Wh-categories are barely separable for frozen embeddings, the prompts generated by $H_\phi$ form well-separated clusters in the LLM hidden space (see token analysis in Sec.~\ref{sec:Experiments}). Combined with the shared $z_{\mathrm{inv}}$, this allows neighbouring task codes to produce visually appropriate prompts under imperfect routing, so end performance is largely insensitive to routing accuracy by design. Consistent with this, retrieval errors are systematic rather than random: when the router errs, it selects a semantically adjacent task rather than an arbitrary one, so a misrouted query receives a code that produces a similar prompt. The next subsection bounds the worst case directly.

\subsection{Task-Code Contribution and the Prompt Pathway}
\label{app:code_swap}
The small oracle gap in Table~\ref{tab:oracle_routing} measures the cost of imperfect routing, not the importance of the task codes. To isolate the contribution of the retrieved code, we replace it at inference with the code of the \emph{most-dissimilar} stored task (lowest key cosine similarity) while keeping the shared component $z_{\mathrm{inv}}$ fixed. Table~\ref{tab:code_swap} shows that this far more extreme substitution than routing ever makes lowers \ACC{} from $64.11$ to $62.14$, confirming that the task-specific component carries measurable task-relevant information; the drop is moderate because transferable structure is intentionally captured by the generator and $z_{\mathrm{inv}}$.

\begin{table}[h]
\centering\small
\begin{tabular}{lc}
\toprule
Code assignment (NExT-QA) & \ACC{} \\
\midrule
Oracle & 64.61 \\
Retrieved (ours) & 64.11 \\
Most-dissimilar task code & 62.14 \\
Prompt pathway disabled ($g_l\!=\!0$) & 25.24 \\
\bottomrule
\end{tabular}
\caption{\textbf{Code-swap and prompt-pathway ablation} on NExT-QA. A wrong code costs $1.97$ points; removing the generated prompt entirely costs $38.87$.}
\label{tab:code_swap}
\end{table}

The prompt pathway itself is load-bearing: zeroing the adapter gates ($g_l\!=\!0$ in Eq.~\ref{eq:gated_prefix}), which removes the generated prompt from every layer, collapses accuracy to $25.24$, a $38.87$-point drop, whereas the most-dissimilar code costs only $1.97$ points. This contrast clarifies the division of roles: the generated prompt carries the transferable structure, while the task-specific code refines it. Consistently, the learned gate modulation is selective and concentrated in a small number of mid-to-late-layer heads, with $\operatorname{mean}(\tanh g_l) \approx 0.01$ and $\max|\tanh g_l| \approx 0.25$ (the per-layer maximum grows from $0.03$ in early layers to $0.25$ in the final layer). The small oracle gap therefore reflects robustness to imperfect routing rather than a lack of task-code contribution.

\subsection{Efficiency Analysis}
\label{app:efficiency}

\begin{table*}[t]
\centering
\caption{\textbf{Computational cost comparison.} Fine-tuned parameters, training time (hours), and peak GPU memory (GB) across continual VideoQA settings. Numbers in parentheses indicate timing without auxiliary losses. For ProgPrompt, peak memory grows across the task stream (initial $\rightarrow$ final).}
\label{tab:compute_cost}
\small
\setlength{\tabcolsep}{4pt}
\renewcommand{\arraystretch}{0.8}
\begin{tabular}{lcccc}
\toprule
\textbf{Metric} & \textbf{Bisecle} & \textbf{ProgPrompt} & \textbf{\model} & \textbf{\model~(1 step)} \\
\midrule
Fine-tuned parameters (M)                          & 3.2M    & 5.1M                       & 9.6M  & 9.6M \\
\midrule
Training time on NExT-QA                           & 1.48    &  2.13                          & 5.55  & 2.46 \\
Peak GPU Mem on NExT-QA (GB)                       & 18.9    & 21.16$\rightarrow$25.43    &  103.8     & 82.0 \\
\midrule
Training time on NExT-QA$\rightarrow$DramaQA       & 5.77    & 6.42                           & 15.58 & 8.03 \\
Peak GPU Mem on NExT-QA$\rightarrow$DramaQA (GB)   & 18.9    & 22.98$\rightarrow$29.34    &   103.9    & 82.1 \\
\bottomrule
\end{tabular}
\end{table*}

Table~\ref{tab:compute_cost} reports trainable parameters, training time, and peak GPU memory. We measure efficiency under matched batch size of four on $2\times$ H200 GPUs, with $\mathcal{L}_{\mathrm{VAQ}}$ active for all methods and $\mathcal{L}_{\mathrm{QAV}}$ omitted since ProgPrompt is incompatible with it. \model~uses $9.6$M trainable parameters, roughly $3\times$ Bisecle and $2\times$ ProgPrompt, yet still under $0.15\%$ of the $7$B-parameter backbone. Training cost and peak memory are correspondingly higher, reflecting two principled design choices: (i) a transformer-based token generator that synthesises layer-specific prompts on demand, and (ii) the look-ahead inner loop in LA-Reg, which doubles forward passes per step. Crucially, reducing LA-Reg to a single look-ahead step (\model{}~$1$-step) more than halves training time ($5.55\!\to\!2.46$h on NExT-QA, $15.58\!\to\!8.03$h on the joint stream) and trims peak memory by $\sim\!20\%$, while still surpassing Bisecle on both metrics ($63.92$ vs $62.37$ \ACC{}, $5.02$ vs $5.34$ \FOG{}).

ProgPrompt's lower wall-clock also stems from a classification objective that sidesteps autoregressive next-token prediction. Its limitation, however, is scalability: memory \emph{grows with the task sequence} ($21.16\!\to\!25.43$ on NExT-QA, $22.98\!\to\!29.34$ on the joint stream) since per-task adapters and prompts must be retained, an overhead that compounds over long continual deployments. \model{}'s footprint, by contrast, remains \emph{constant} across the task stream: each new task adds only kilobytes through the compact task code, independent of stream length.

Overall, \model{}'s additional compute is a one-time investment in the generator and look-ahead optimisation, repaid on three fronts: (i) consistent SOTA accuracy and forgetting across single-dataset, multi-dataset, and ImageQA$\rightarrow$VideoQA settings; (ii) stronger zero-shot generalisation to unseen datasets at test time; and (iii) constant per-task memory growth that scales gracefully to long task streams---a property whose value only increases as deployments accumulate more tasks.

\paragraph{LA-Reg at long horizons.}
LA-Reg scales linearly with the number of stored task codes, but each term requires only a forward pass through the small generator $H_\phi$ (${\approx}6.4$M parameters), not the 7B LLM. Table~\ref{tab:lareg_overhead} reports a micro-benchmark of the per-step overhead: at $50$ stored codes it is $9.47$\,ms, i.e., $2.2\%$ of a full training step ($438.9$\,ms). The cost can also be made independent of the horizon by sampling a fixed number of stored codes per step: on the 13-task NExT-QA$\rightarrow$DramaQA stream, sampling only $4$ codes yields $61.19$ \ACC{} / $11.96$ \FOG{} versus $61.96$ / $11.72$ with all codes, a $0.77$-point reduction that remains far above Bisecle ($55.31$ \ACC{}). Regarding plasticity, since the regulariser sums over $t-1$ codes, averaging over stored codes (equivalently $\lambda_{\mathrm{LA\text{-}Reg}} \propto \tfrac{1}{t-1}$) keeps its per-task strength constant; our fixed setting nonetheless shows no degradation in new-task learning over the 13-task stream.

\begin{table}[h]
\centering\small
\begin{tabular}{lccc}
\toprule
Stored codes & 12 & 25 & 50 \\
\midrule
LA-Reg overhead per step (ms) & 2.74 & 5.05 & 9.47 \\
\bottomrule
\end{tabular}
\caption{\textbf{LA-Reg overhead vs.\ number of stored codes} (per training step; full step $= 438.9$\,ms).}
\label{tab:lareg_overhead}
\end{table}

\paragraph{Joint-training upper bound.}
To place the gains in context, Table~\ref{tab:joint_bound} reports joint (non-continual) training on all eight NExT-QA tasks under the same backbone and training budget. \model{} reaches $93.7\%$ of the joint-training accuracy versus $91.2\%$ for Bisecle; equivalently, it closes about $29\%$ of the gap between Bisecle and the non-continual ceiling.

\begin{table}[h]
\centering\small
\begin{tabular}{lc}
\toprule
Method (NExT-QA) & \ACC{} \\
\midrule
Joint training (upper bound) & 68.42 \\
\model{} & 64.11 \\
Bisecle & 62.37 \\
\bottomrule
\end{tabular}
\caption{\textbf{Joint-training upper bound} on NExT-QA.}
\label{tab:joint_bound}
\end{table}

\subsection{Sensitivity to the Number of Prompt-Injection Layers}

\begin{table}[h]
\centering
\caption{Performance with different numbers of prompt layers.}
\label{tab:layer_ablation}
\small
\setlength{\tabcolsep}{4pt}
\renewcommand{\arraystretch}{0.95}
\begin{tabular}{c|cc}
\toprule
\# P. Layer & \HACC & \HFOG \\
\midrule
8  & 54.32 & 9.02 \\
\rowcolor{black!3} 16 & 58.30 & 7.13 \\
24 & \runner{59.04} & \runner{6.87} \\
\rowcolor{black!3} 32 & \best{64.11} & \best{4.76} \\
\bottomrule
\end{tabular}
\end{table}

We study how \model~ varies with the number of adapted adapter layers in Table~\ref{tab:layer_ablation}. Overall, increasing the adapted depth consistently improves performance, supporting the effectiveness of parameter-efficient adaptation for continual VideoQA.

\subsection{Task Order}

We evaluate \model~under different task orders on NExT-QA and compare against Bisecle in Table~\ref{tab:task_order_performance}. \model~is more robust, outperforming Bisecle on three orders while remaining competitive on the fourth, supporting the effectiveness of multimodal token hyper-generation.

\begin{table*}[!ht]
    \centering
    \caption{Robustness to task order on NExT-QA: Bisecle versus \model~, reporting average accuracy and forgetting across four task permutations.}
    \label{tab:task_order_performance}
    \resizebox{0.9\textwidth}{!}{
        \begin{tabular}{lcccc}
            \toprule
            \multirow{2}{*}{Task Order} & \multicolumn{2}{c}{\HACC} & \multicolumn{2}{c}{\HFOG} \\
            \cmidrule(lr){2-3} \cmidrule(lr){4-5}
             & Bisecle & \model~ & Bisecle & \model~ \\
            \midrule
            $<$TP, TN, CH, TC, DL, DO, CW, DC$>$ & 57.98 & \best{61.23} & 7.16 & \best{6.82} \\
            \rowcolor{black!3}$<$DO, CW, DC, CH, TP, TC, TN, DL$>$ & 57.95 & \best{61.35} & 8.93 & \best{6.91} \\
            $<$CW, DO, TN, DL, TC, TP, DC, CH$>$ & 62.25 & \best{62.79} & 5.70 & \best{4.96} \\
            \rowcolor{black!3}$<$CH, DL, TP, TC, DC, DO, TN, CW$>$ & \best{63.09} & 62.96 & 2.87 & \best{2.79} \\
            \bottomrule
        \end{tabular}
    }
\end{table*}

\section{Qualitative Results}
\label{sec:more_qualitative_results}

\subsection{Continual VideoQA}

Fig.~\ref{fig:qual_vqa_12} and Fig.~\ref{fig:qual_vqa_34} present qualitative comparisons in continual VideoQA. \model~consistently yields the correct answer, whereas Bisecle fails in these cases.

\begin{figure*}[t]
  \centering
  \begin{subfigure}[t]{0.49\textwidth}
    \centering
    \includegraphics[width=\linewidth]{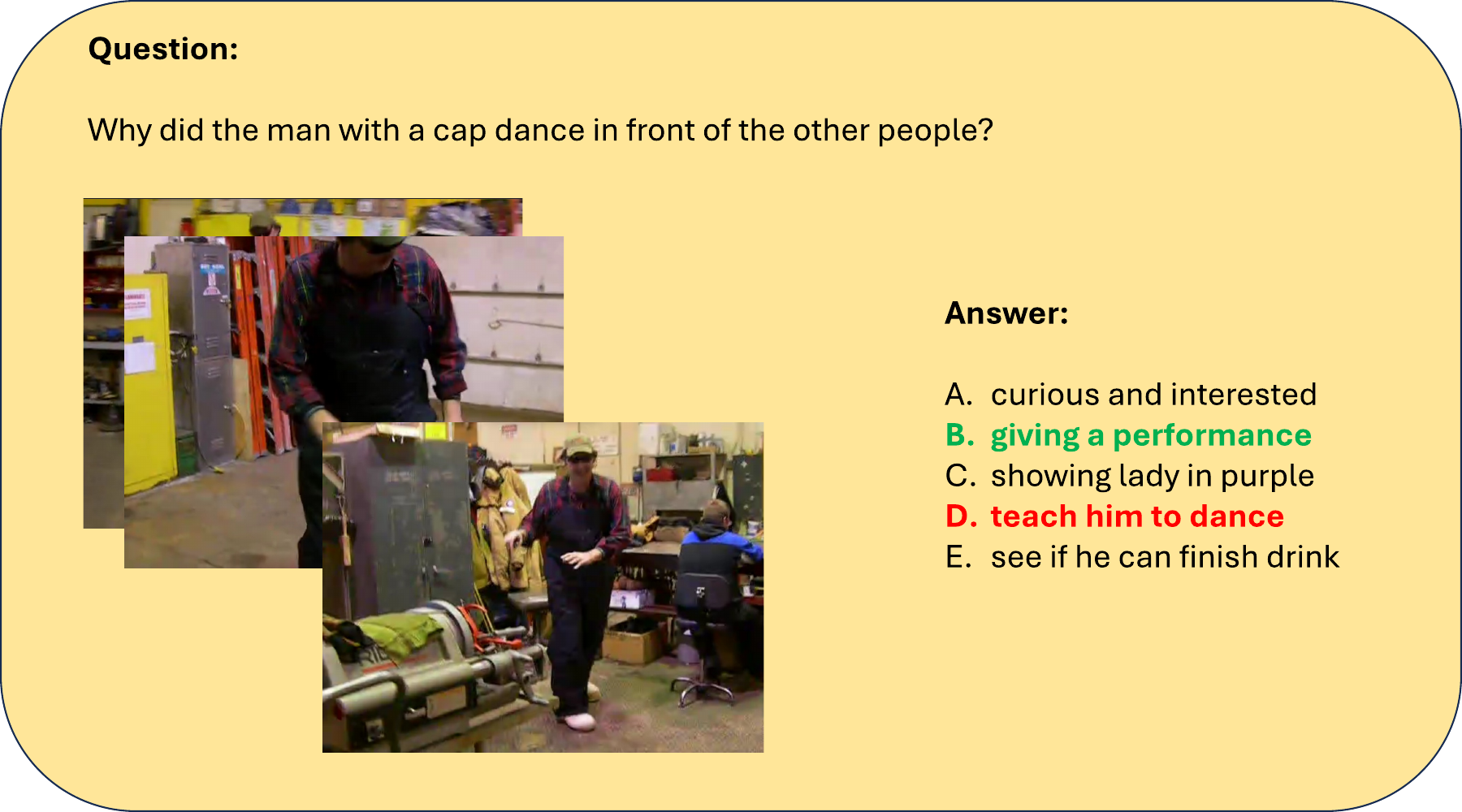}
  \end{subfigure}\hfill
  \begin{subfigure}[t]{0.49\textwidth}
    \centering
    \includegraphics[width=\linewidth]{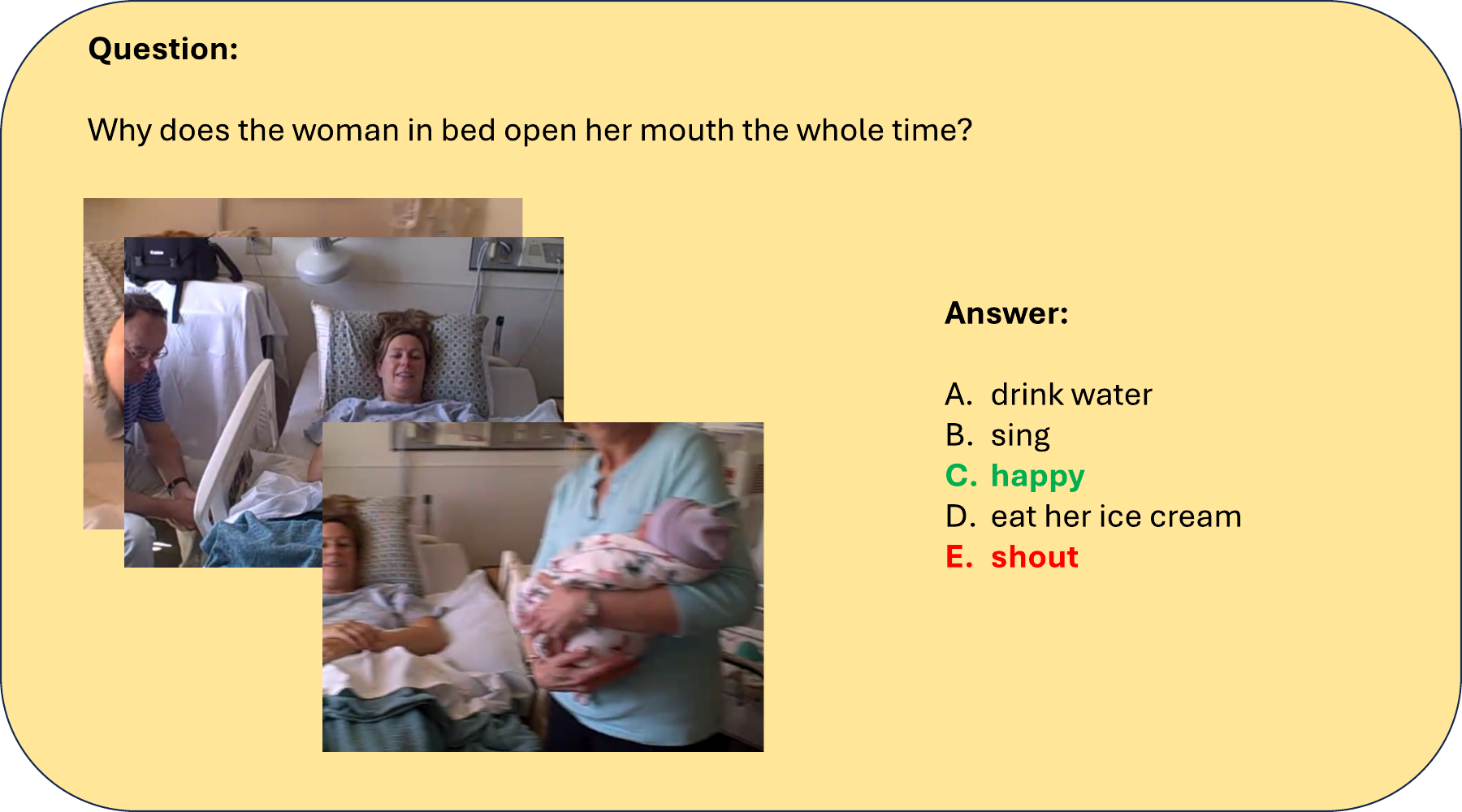}
  \end{subfigure}
  \caption{\textbf{Qualitative Examples 1--2.} \model~ predicts the correct answer (\textcolor{green!55!black}{green}), whereas Bisecle produces an incorrect one (\textcolor{red!70!black}{red}) in continual VideoQA.}
  \label{fig:qual_vqa_12}
\end{figure*}

\begin{figure*}[t]
  \centering
  \begin{subfigure}[t]{0.49\textwidth}
    \centering
    \includegraphics[width=\linewidth]{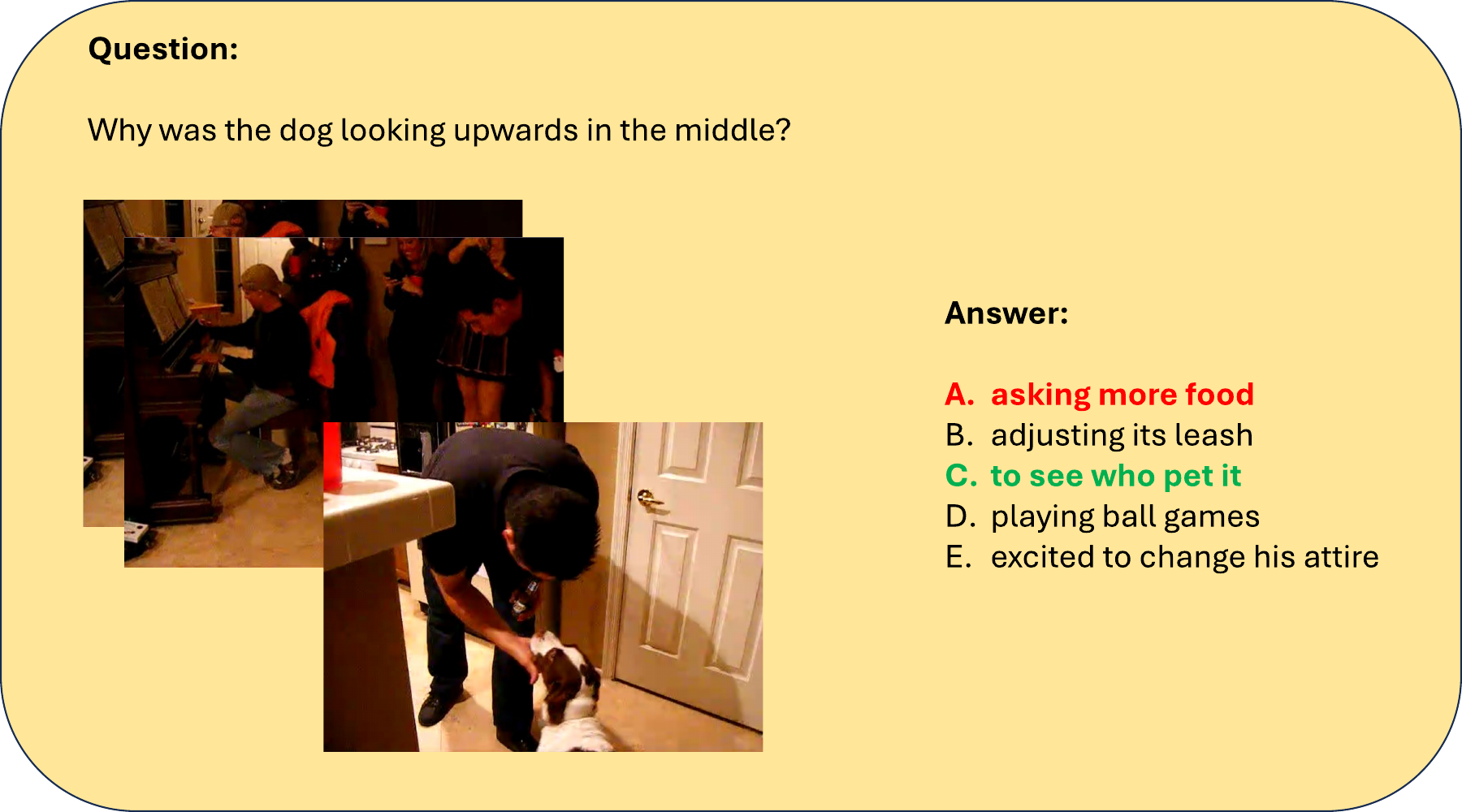}
  \end{subfigure}\hfill
  \begin{subfigure}[t]{0.49\textwidth}
    \centering
    \includegraphics[width=\linewidth]{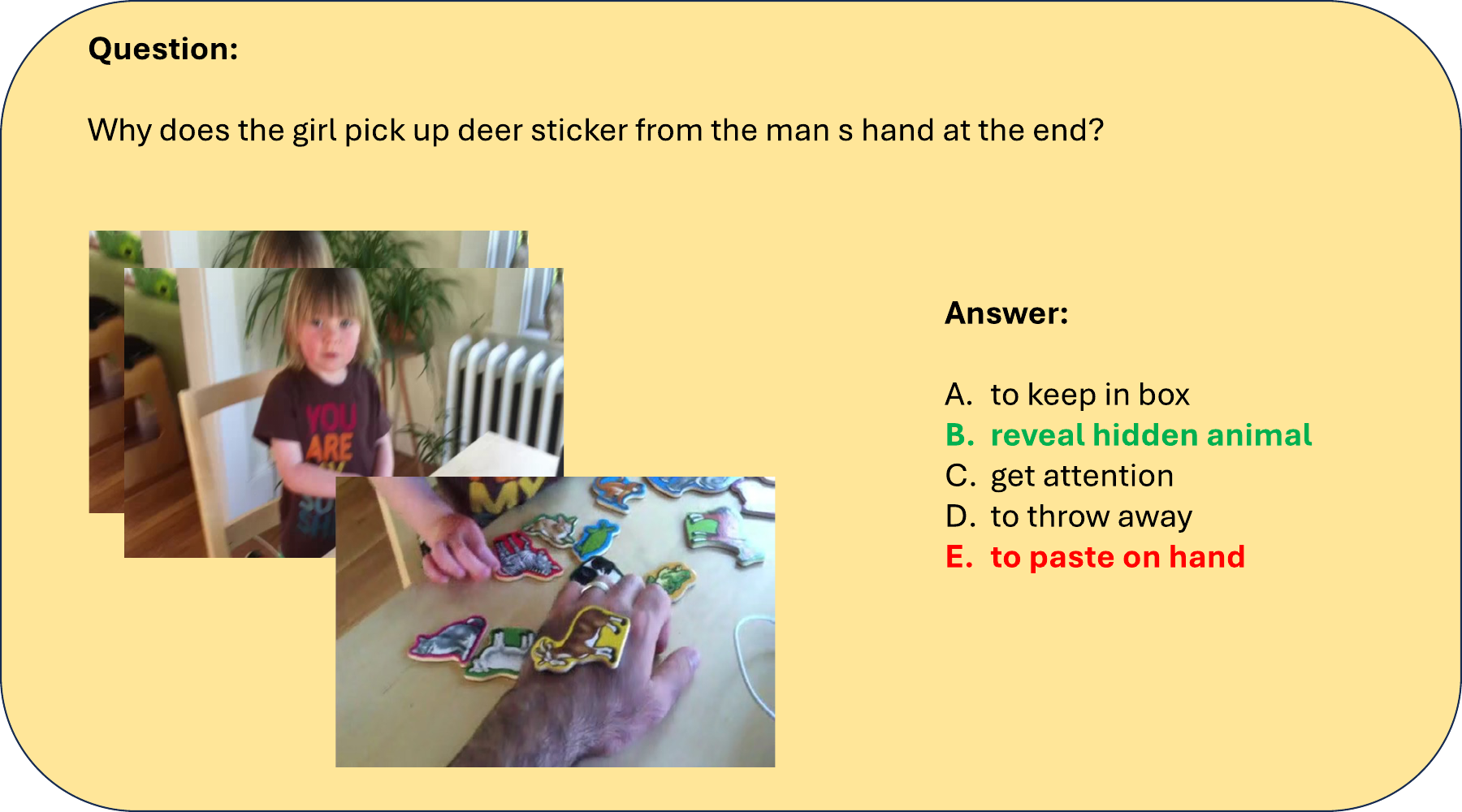}
  \end{subfigure}
  \caption{\textbf{Qualitative Examples 3--4.} \model~ predicts the correct answer (\textcolor{green!55!black}{green}), whereas Bisecle produces an incorrect one (\textcolor{red!70!black}{red}) in continual VideoQA.}
  \label{fig:qual_vqa_34}
\end{figure*}

\subsection{ImageQA$\rightarrow$VideoQA transfer}

Fig.~\ref{fig:qual_vqa_56} show qualitative results for continual ImageQA$\rightarrow$VideoQA transfer. After adapting to video tasks, \model~ still predicts the correct answer, whereas Bisecle fails on these Visual7W examples.

\begin{figure*}[t]
  \centering
  \begin{subfigure}[t]{0.49\textwidth}
    \centering
    \includegraphics[width=\linewidth]{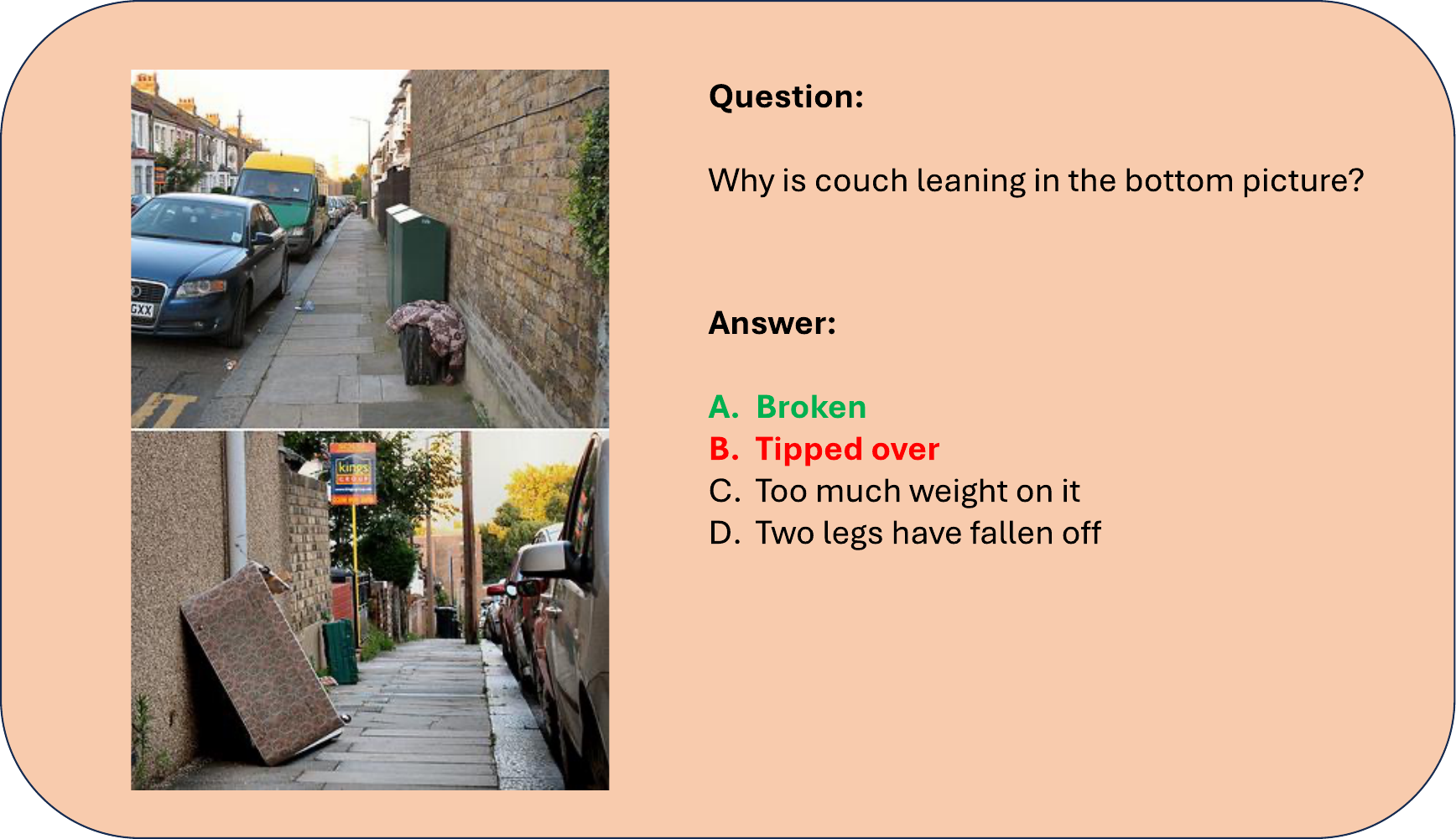}
  \end{subfigure}\hfill
  \begin{subfigure}[t]{0.49\textwidth}
    \centering
    \includegraphics[width=\linewidth]{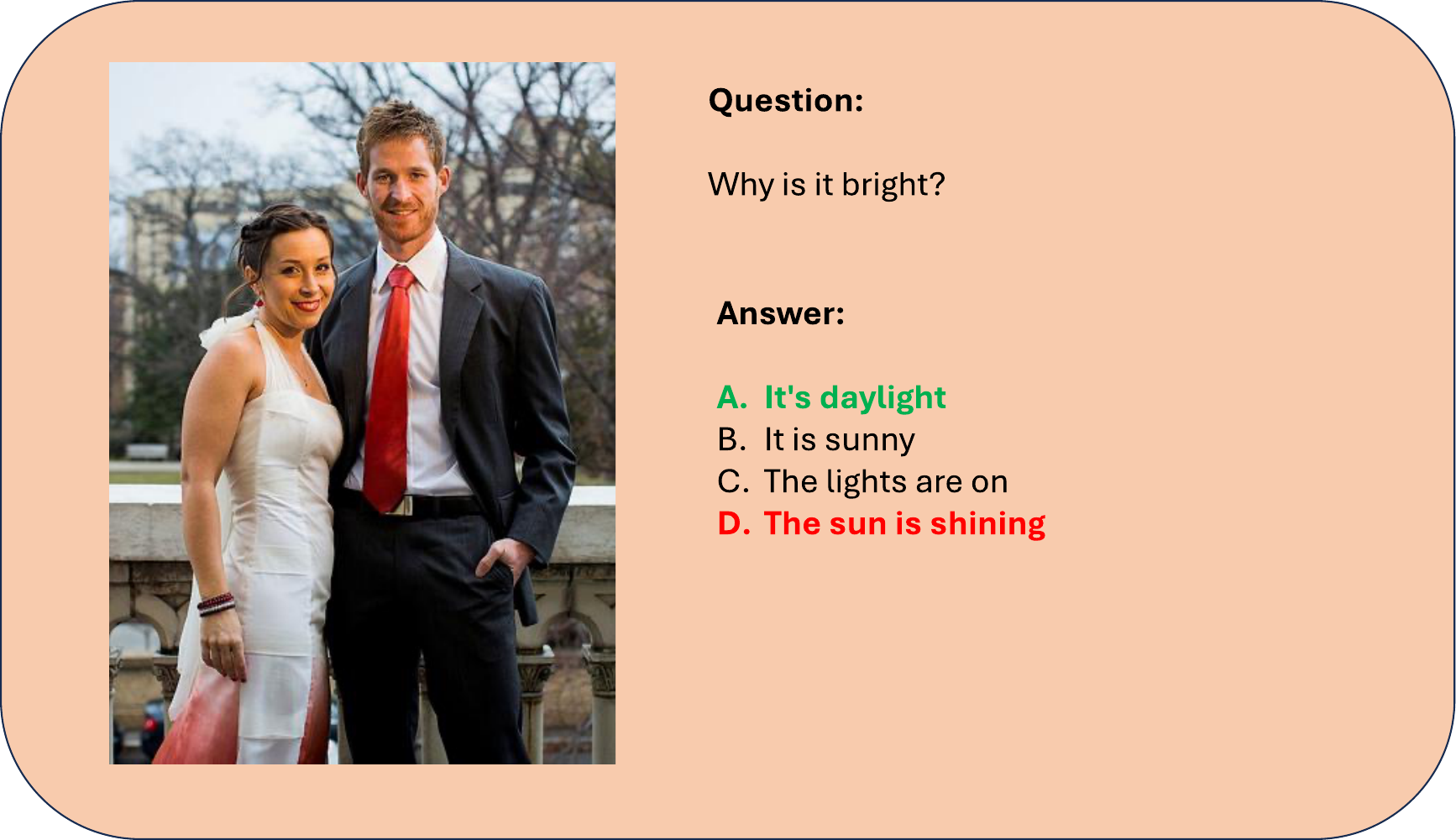}
  \end{subfigure}
  \caption{\textbf{Qualitative Examples 5--6.} After adapting to video tasks, \model~ still predicts the correct answer (\textcolor{green!55!black}{green}), whereas Bisecle fails on these Visual7W examples (\textcolor{red!70!black}{red}).}
  \label{fig:qual_vqa_56}
\end{figure*}

% \begin{table}[!t]
% \centering
% \caption{\emph{Continual VideoQA} results on \STARcap~with accuracy and forgetting.
% \bestcap{Bold} and \runnercap{underline} denote the best and second-best results.}
% \label{tab:videoqa_results_star}

% \begingroup
% % Make \venue smaller ONLY inside this table
% \renewcommand{\venue}[1]{\textcolor{black!60}{\textsf{(\fontsize{6}{6.8}\selectfont #1)}}}

% \fontsize{7}{8}\selectfont
% \setlength{\tabcolsep}{2.5pt}
% \renewcommand{\arraystretch}{0.9}

% \resizebox{0.5\textwidth}{!}{%
% \begin{tabular}{L{3.0cm}*{2}{c}}
% \toprule
% \textbf{Method ~\venue{Venue}} &
% \hdr{STARc}{STAR} \\
% & \HACC & \HFOG \\
% \midrule
% \rowcolor{black!3}{LLaMA-Adapter ~\venue{ICLR'24}} 
% & 46.89 & 11.54 \\
% \midrule
% L2P ~\venue{CVPR'22}        
% & 48.25 & 10.82 \\
% \rowcolor{black!3}DualPrompt ~\venue{ECCV'22} 
% & 49.73 & 10.11 \\
% LAE ~\venue{ICCV'23}         
% & 49.15 &  9.87 \\
% \rowcolor{black!3}ProgPrompt ~\venue{ICLR'23}  
% & 51.05 &  8.75 \\
% ColPro ~\venue{EMNLP'24}     
% & 48.67 &  8.13 \\
% \rowcolor{black!3}DAM ~\venue{WACV'25}        
% & 50.64 &  8.92 \\
% Bisecle ~\venue{NeurIPS'25}     
% & \best{52.16} & \runner{7.60} \\
% \midrule
% \rowcolor{black!3}\textbf{\model~ ~\venue{ours}} 
% & 48.31 & \best{5.92} \\
% \bottomrule
% \end{tabular}
% }% end resizebox

% \endgroup
% \end{table}

% Plan to improve \model~

% 1. Fix Routing Error
% 2. Rerun with correct version
% 3. Update the paper with feedback from reviewing.

\end{document}